\RequirePackage{amsthm}
 
\documentclass[sn-mathphys-ay, iicol]{sn-jnl}% Math and Physical Sciences Author Year Reference Style
\usepackage{graphicx}%
\usepackage{multirow}%
\usepackage{amsmath,amssymb,amsfonts}%
\usepackage{amsthm}%
\usepackage{wasysym}
\usepackage{mathrsfs}%
\usepackage[title]{appendix}%
\usepackage{xcolor}%
\usepackage{textcomp}%
\usepackage{manyfoot}%
\usepackage{booktabs}%
\usepackage{algorithm}%
\usepackage{algorithmicx}%
\usepackage{algpseudocode}%
\usepackage{listings}%
\usepackage{anyfontsize}
\usepackage{stackengine}
\usepackage{scalerel}
\usepackage{adjustbox}

\theoremstyle{thmstyleone}%
\newtheorem{theorem}{Theorem}%  meant for continuous numbers
\newtheorem{proposition}[theorem]{Proposition}% 
\newtheorem{lemma}[theorem]{Lemma}
\newtheorem{corollary}[theorem]{Corollary}

\theoremstyle{thmstyletwo}%
\newtheorem{example}{Example}%
\newtheorem{remark}{Remark}%

\theoremstyle{thmstylethree}%
\newtheorem{definition}{Definition}%

\def\tang{\ThisStyle{\abovebaseline[0pt]{\scalebox{-1}{$\SavedStyle\perp$}}}}

\newcommand{\zb}[1]{\ensuremath{\boldsymbol{#1}}}

\usepackage[capitalize,noabbrev]{cleveref}

\DeclareMathOperator*{\argmin}{arg\,min}

\newcommand{\dd}{\textnormal{d}}
\newcommand{\dist}{\textnormal{dist}}

\newcommand{\diff}
{\textnormal{Diff}}

\newcommand{\fcdot}{\; \cdot\; }
\newcommand{\Hsp}{\mathcal{H}(V, M)}
\newcommand{\Fsp}{\mathcal{F}(V, M)}
\newcommand{\step}{\textnormal{step}}

\newcommand\blfootnote[1]{%
  \begingroup
  \renewcommand\thefootnote{}\footnote{#1}%
  \addtocounter{footnote}{-1}%
  \endgroup
}

\begin{document}

\title[Article Title]{Manifold GCN: Diffusion-based Convolutional Neural Network for Manifold-valued Graphs}


\author*[1,2,3]{\fnm{Martin} \sur{Hanik}}\email{hanik@zib.de}

\author[2]{\fnm{Gabriele} \sur{Steidl}}\email{steidl@math.tu-berlin.de}

\author[3]{\fnm{Christoph} \spfx{von} \sur{Tycowicz}}\email{vontycowicz@zib.de}

\author[+]{for~the~Alzheimer’s~Disease~Neuroimaging~Initiative}

\affil*[1]{\orgname{BIFOLD—Berlin Institute for the Foundations of Learning and Data}, \orgaddress{\street{Ernst-Reuter Platz 7}, \city{Berlin}, \postcode{10587}, \country{Germany}}}

\affil[2]{\orgname{Technical University Berlin}, \orgaddress{\street{Straße des 17. Juni 136}, \city{Berlin}, \postcode{10623}, \country{Germany}}}

\affil[3]{\orgname{Zuse Institute Berlin}, \orgaddress{\street{Takustraße 7}, \city{Berlin}, \postcode{14195}, \country{Germany}}}


\abstract{We propose two graph neural network layers for graphs with features in a Riemannian manifold.
First, based on a manifold-valued graph diffusion equation, we construct a diffusion layer
that can be applied to an arbitrary number of nodes and graph connectivity patterns.
Second, we model a tangent multilayer perceptron by transferring ideas from the  vector neuron framework 
to our general setting.
Both layers are equivariant under node permutations and the feature manifold's isometries. These properties have led to a beneficial inductive bias in many deep-learning tasks. Furthermore, they enable novel, more flexible feature designs.
Numerical examples on synthetic data and an Alzheimer's classification application on triangle meshes of the right hippocampus demonstrate the usefulness of our new layers: While they apply to a much broader class of problems, they outperform task-specific state-of-the-art networks.}




\maketitle

\blfootnote{${}^{+}$Data used in the preparation of this article was obtained from the Alzheimer’s Disease
Neuroimaging Initiative (ADNI) database (adni.loni.usc.edu). As such, the investigators
within the ADNI contributed to the design and implementation of ADNI and/or provided data
but did not participate in the analysis or writing of this report. A complete listing of ADNI
investigators can be found at: \url{http://adni.loni.usc.edu/wp-content/uploads/how_to_apply/ADNI_Acknowledgement_List.pdf}
}

%-------------------------------------------------
\section{Introduction}
%-------------------------------------------------
Graph neural networks (GNNs) have gained widespread popularity for the analysis of graph-structured data. 
They have found applications in various areas, such as 
bioinformatics~\citep{torng2019graph,Zhang_ea2021,yi2022graph}, 
physics~\citep{sanchez2018graph,pmlr-v119-sanchez-gonzalez20a,shlomi2020graph}, 
and social sciences~\citep{fan2019graph,wu2020graph,kumar2022influence,wu_ea2023}.  
GNNs utilize the features attributed to the graph's nodes as an important source of information.
Here, the vast majority of the GNNs rely on an Euclidean feature space. 
However, there are many applications where the data lies in a non-flat manifold; for example, 
on the sphere~\citep{fisher1993statistical,de2022riemannian,nava2023sasaki}, the special orthogonal group SO(3) \citep{BCHPS2016,GNHSLP2022},
the space of symmetric positive definite (SPD) matrices~\citep{cherian2016positive,FLETCHER2007250,YOU2021117464,pennec2020manifold,hanik2021predicting,wong2018riemannian,ju2023graph,nerrise2023explainable}, 
the Stiefel manifold~\citep{lin2017bayesian,chakra2019statStiefel,chen2020learning,mantoux2021understanding}, 
and 
the Grassmannian manifold~\citep{turaga2011statistical,huang2015projection,huang2018building}; also, Lie Group-valued features are 
of interest~\citep{ParkBobrowPloen1995,hanik2022bi,vemulapalli2014human,vemulapalli2016rolling,vonTycowicz_ea2018,AmbellanZachowTycowicz2021}. Even more, Riemannian feature modeling is considered to be a unifying paradigm~\citep{sun2026riemanngl} that could be necessary for
advancing towards graph foundation models~\citep{yu2026riemannian}.

So far, there exist only networks that are designed for particular types of manifolds, such as (products of) spaces of constant curvature~\citep{bachmann2020constant,zhu2020graph, sun2022self,deng2023fmgnn}
or spaces of SPD matrices~\citep{huang2017riemannian,ju2023graph,zhao2023modeling}, and these architectures cannot deal with \textit{general} manifold-valued features.
    
This paper presents a GNN architecture that can work with data from \textit{any} Riemannian manifold 
for which geodesics can be efficiently approximated. Contrary to existing GNNs, it can thus also be applied to other feature manifolds than spheres, hyperbolic spaces, and SPD matrices, such as rotation matrices and shape spaces.
The practical advantages of the network are demonstrated through two classification experiments. Remarkably, our networks outperform state-of-the-art methods, while being more generally applicable than their manifold-specific counterparts.

The core element of our architecture is a convolution-type layer that discretizes \textit{graph diffusion in manifolds}. We thereby utilize a link between convolution and diffusion processes that GNNs have successfully used for Euclidean data~\citep{atwood2016diffusion, li2018diffusion, gasteiger2019diffusion}.
Combined with an explicit solution scheme for the diffusion equation, the \emph{new diffusion layer} exhibits the same local aggregation of features that makes convolution layers so successful~\citep{bronstein2021geometric}. Furthermore, it is equivariant under node permutations and the feature space isometries. 
These properties (approximately) hold for many relationships that deep learning techniques are intended to capture~\citep{bronstein2021geometric} and have been found to benefit the learning~\citep{satorras2021n,unke2021se,gerken2023geometric,yang2023alphafold2}. To the best of our knowledge, we are the first to develop GNN layers for manifolds that exhibit equivariance under isometries.

Existing models~\citep{kipf2017semisupervised,bronstein2021geometric,sharp2022diffusionnet,dai2021hyperbolic} can account for the non-Euclidean structure of the
domain (that is, graphs or surfaces); however, except in very special cases, they fail to exploit the regularities of the co-domain. Being faithful to the geometry of both provides a strong inductive bias and enables novel feature designs that diverge considerably from previous methodologies. In our experiments, we demonstrate that our model improves performance by utilizing feature manifolds beyond constant or non-positive curvature spaces and by employing novel embedding strategies for abstract graphs made possible by feature symmetry equivariance.

In addition to the diffusion layer, we propose a novel \textit{tangent multilayer perceptron} that can be seen as a generalization of a fully connected multilayer perceptron to manifolds, including nonlinearities between layers. It shares the symmetries of the diffusion layer.  
    
We test a network incorporating our novel layers on two classification tasks: a benchmark on synthetic graphs and a classification of Alzheimer's disease from triangle meshes of the right hippocampus. Our method is as good as or better than its competition in both cases. Our layers lead to fast learning, even with smaller data sets.
The implementation of our network units is available as part of the open-source \texttt{Morphomatics} library~\citep{Morphomatics}. Data and code to reproduce the experiments can be found at \url{https://github.com/morphomatics/MfdGCN}.

\noindent
\emph{Outline of the paper.}
We start by recalling related work in Section~\ref{sec:related_work}.
Then, in Section~\ref{sec:background}, we provide the necessary background on Riemannian manifolds and graph Laplacians
with manifold-valued features.
In Section \ref{sec:diff_layer}, we introduce our novel diffusion layer and prove its equivariance properties.
Our second layer, the tangent multilayer perceptron, is constructed in Section~\ref{sec:tLL} 
and shows the same desirable equivariance behavior as the diffusion layer.
Both layers are combined within a generic graph convolutional neural network (GCN) block in Section~\ref{sec:GCN}.
Section~\ref{sec:numerics} contains our numerical results.
Finally, Section~\ref{sec:conclusions} gives a summary and conclusions for future work.
The appendices contain further theoretical results on the diffusion layer.

%--------------------------------------------------------
\section{Related Work} \label{sec:related_work}
%--------------------------------------------------------
This section reviews related works and highlights the differences in our contribution. 
We distinguish between two types of architecture:
First, we start with an overview of GNNs for graphs \textit{with Euclidean features} 
that discretize a continuous flow in the feature space. (We refer to \citep{wu2020comprehensive} 
for an overview of the full landscape of Euclidean GNNs.)
Then, we collect deep learning architectures that can handle \textit{manifold-valued} features. 

\subsection{Flow-based GNNs in the Euclidean Space}

\subsubsection{Diffusion networks}
\cite{atwood2016diffusion} and \cite{gasteiger2019diffusion} used diffusion to construct powerful (graph) convolutional neural networks. 
Several papers improved and modified the basic idea: 
Adaptive support for diffusion-convolutions was proposed by~\cite{zhao2021adaptive}; implicit nonlinear diffusion was introduced to capture long-range interactions by~\cite{chen2022optimization}; and
\cite{liao2019lanczosnet} used a polynomial filter on the graph Laplacian that corresponds to a multi-scale diffusion equation.
    
Also based on a graph diffusion equation, a broad class of well-performing GNNs was identified by~\cite{chamberlain2021grand}. Shortly afterward, \cite{thorpe2021grand++} showed that adding a source term to the equation is helpful in specific scenarios. 
\cite{chamberlain2021beltrami} additionally considered a nonlinear extension of the graph diffusion equation, and found that many known GNN architectures are instances of the resulting class of GNNs. 
All these approaches allow using non-Euclidean distances for data \textit{encoding}. 
However, in contrast to our method, they never treat the features as intrinsic to a manifold.

A different line of work uses diffusion processes in the context of cellular sheaf theory for 
graph learning~\citep{hansen2020sheaf,bodnar2022neural}.
    
Finally, a diffusion-based architecture especially for graphs that discretize 2-manifolds (for example, triangular surface meshes) was proposed by~\cite{sharp2022diffusionnet}. While the underlying idea of using diffusion for information aggregation is quite similar to our method, it can only deal with (several) scalar functions defined over a discrete surface. In particular, it cannot handle abstract graphs that are not embedded in Euclidean space and do not come with a notion of tangent plane and gradient; furthermore, it can only erroneously interpret manifold-valued features as collections of scalar functions.

%-------------------------------------------------------------------
\subsubsection{Neural ordinary differential equations}
    An approach similar to diffusion-based GNNs is the neural ordinary differential equations (neural ODE) framework~\citep{chen2018neural}. Instead of diffusion-based methods, which treat GNNs as particular instances of the discretization of a \textit{partial} differential equation in both space and time~\citep{chamberlain2021beltrami}, it discretizes an underlying ODE. The idea was transferred to graph learning by~\cite{poli2019graph} and \cite{xhonneux2020continuous}. 
    
    The discrete nature of the underlying graph leads to the fact that, although motivated by a PDE, our network is built on a set of ODEs (one for each node). Consequently, our proposed method also fits into the neural ODE framework. 

%------------------------------------------------------------------------------------------------------------
\subsection{Deep Neural Networks for Manifold-valued Signals}
 Here, we distinguish between three approaches.

\subsubsection{Networks for manifold-valued grids}
    ManifoldNet~\citep{chakraborty2020manifoldnet} is a convolutional neural network that can take manifold-valued images as input. Convolutions in manifolds are generalized through weighted averaging. However, in contrast to our approach, their layers can only work on \textit{regular grids} as an underlying structure and not on general graphs. A similar convolutional approach using diffusion means was presented by~\cite{sommer2020horizontal}.

\subsubsection{GNNs for special manifolds} \label{sec:gnns_manifolds}
    There have been two primary motivations for building GNNs that can deal with manifold-valued features: first, embedding abstract graphs in curved manifolds to utilize their geometric structure for downstream tasks, and second, learning from interrelated measurements that take values in some manifold. 
    Surprisingly, most existing work focuses on the first aim. In both cases, though, the resulting GNNs can only handle data from very restricted classes of manifolds. We discuss the networks and their associated spaces in the following.
        
    There is mounting empirical evidence that embedding graphs in non-Euclidean spaces can help with various tasks. For instance, \cite{krioukov2010hyperbolic} observed that a hyperbolic representation is beneficial when dealing with graphs with a hierarchical structure. Later, products of constant-curvature spaces appeared as appropriate embedding spaces for more classes of graphs~\citep{gu2018learning}. 
    In graph representation learning, these observations were further backed up by specific constructions that successfully embed (abstract) graphs into hyperbolic space for downstream tasks~\citep{chami2019hyperbolic,liu2019hyperbolic, dai2021hyperbolic,zhang2021lorentzian, pmlr-v202-yang23u,van2023poincare}. 
    Further studies showed that other spaces can be better suited than hyperbolic ones for some learning tasks on broader classes of graphs.
    While (products of) spaces of constant curvature have been used frequently~\citep{bachmann2020constant,zhu2020graph, sun2022self,deng2023fmgnn,sun2024motif,xue2024residual}, also spaces of non-constant curvature have been utilized successfully in the form of Grassmannians~\citep{cruceru2021computationally} and SPD~\citep{cruceru2021computationally,zhao2023modeling} manifolds. 
    \cite{xiong2022pseudo} also employed pseudo-Riemannian manifolds as embedding spaces in the form of pseudo-hyperboloids.

    All the above approaches are aimed at embedding abstract graphs for downstream tasks. They thus have the common assumption that the \textit{initial} embedding into the manifold can be learned. 
    By contrast, \cite{ju2023graph} tackled an application where the graphs already come with an initial embedding: They considered
    SPD-valued graphs that naturally arise from EEG imaging data, and developed network blocks that could navigate the space of SPD matrices with the affine invariant metric. This led to excellent classification results.

    Unfortunately, all the mentioned GNN architectures have in common that they are not equivariant or invariant under the isometries of the manifold. Whenever the function-to-be-learned (approximately) possesses one of these properties, this can lead to suboptimal performance: On the one hand, non-invariant or non-equivariant networks tend to have many ``unnecessary'' parameters that make successful training more difficult. (We refer to the discussion of~\cite{bronstein2021geometric} on the curse of dimensionality in machine learning and how shift-invariant/equivariant CNNs helped to overcome it.)
    On the other hand, the training data must explicitly contain the transformations to enable the network to ``discover'' the equivariance/invariance itself~\citep{gerken2023geometric}.

\subsubsection{Flow-based networks on manifolds}
    Neural ODEs were transferred to manifolds in~\cite{lou2020neural} and \cite{katsman2023riemannian}. 
    Besides that, there are also deep learning approaches via discrete~\citep{rezende2020normalizing} and continuous~\citep{mathieu2020riemannian,rozen2021moser,chen2024flow} normalizing flows in Riemannian manifolds. For an overview of normalizing flows and their generalizations for Euclidean data, we refer, for example, to~\cite{HHG2023} and \cite{RH2021}.
    In contrast to our approach, those authors aim to solve a given ODE to learn intricate probability distributions instead of using the ODE to build convolutional layers. The same is true for generative diffusion models~\citep{de2022riemannian,thornton2022riemannian}.

%----------------------------------------------------------------------
\section{Background} \label{sec:background}
%----------------------------------------------------------------------
    In this section, we recall the necessary background from differential geometry (see, for example,~\cite{Petersen2006} for more), the notation of graph Laplacians for manifold-valued graphs by~\cite{bergmann2018graph}, and the definitions of invariance and equivariance. 

%----------------------------------------------------------------------		
\subsection{Riemannian Geometry}
    A Riemannian manifold is a $d$-dimensional manifold\footnote{Also, when not mentioned explicitly, we always assume that manifolds and all maps between manifolds are smooth (that is, infinitely often differentiable).} $M$ together with a Riemannian metric $\langle \fcdot, \fcdot \rangle$. The latter assigns to each tangent space $T_pM$ at $p \in M$ an inner product $\langle \fcdot, \fcdot \rangle_p$ that varies smoothly in $p$. It induces a norm $\| \fcdot\|_p$ on each tangent space $T_pM$ and a distance function $\dist: M \times M \to \mathbb R_{\ge 0}$. Furthermore, it determines the so-called Levi-Civita connection $\nabla$, allowing vector field differentiation. For vector fields $X$ and $Y$ on $M$, we denote the derivative of $Y$ along $X$ by $\nabla_X Y$.
  
    A geodesic is a curve $\gamma: I \to M$ on an interval $I \subseteq \mathbb{R}$ without tangential acceleration, that is,
    $\nabla_{\gamma'}\gamma' = 0,$
    where $\gamma' := \frac{\dd}{\dd t} \gamma$. 
    The manifold $M$ is complete if every geodesic can be defined on the whole $\mathbb{R}$.
    In the rest of the paper, we consider only complete, connected Riemannian manifolds $M$.
    Locally, geodesics are the shortest paths, and the length of a geodesic connecting two points $p, q \in M$ is equal to the distance $\dist(p,q)$.
    Importantly, every point in $M$ has a so-called normal convex neighborhood $U \subseteq M$ in which any pair $p,q \in U$ can be joined by a unique length-minimizing geodesic $\gamma: [0,1] \to M$ with $\gamma(0)=p$ and $\gamma(1)=q$ that lies entirely in $U$.
    We will need the Riemannian exponential map: for $X \in T_pM$, let $\gamma_X$ be the geodesic with $\gamma(0) = p$ and $\gamma'(0) = X$. Since geodesics are solutions to ordinary differential equations, there is a neighborhood $W \subset T_pM$ of $0 \in T_p M$ on which the so-called exponential map
    \begin{align*}
        \exp_p &: W \to M, \quad  X \mapsto \gamma_X(1)
    \end{align*}
    at $p$ is well defined.
    It can be shown that $\exp_p$ is a local diffeomorphism. 
    Let $D_p \subset T_pM$ be the maximal neighborhood of $0 \in D_p$, where this is the case, and set $D_pM:= \exp_p(D_p)$. 
    Then the inverse $\log_p: D_pM \to D_p$ of $\exp_p$ is defined and is called the Riemannian logarithm at $p$. 
    The points for which $\log_p$ is not defined constitute the so-called cut locus of $p$.
    The product $M^n$ is again a Riemannian manifold, everything working component-wise. 

    An important class of maps is the group of isometries (or symmetries) 
    of $M$. These are diffeomorphisms $\Phi: M \to M$ that conserve the Riemannian metric; that is, for every $X, Y \in T_pM$, $p \in M$, we have 
    $$\langle \dd_p\Phi(X), \dd_p\Phi(Y) \rangle_{\Phi(p)} = \langle X, Y \rangle_p,$$ 
    where $\dd_p\Phi: T_pM \to T_{\Phi(p)}M$ is the differential of $\Phi$ at $p$.
    Isometries preserve the distance between any two points, map geodesics into geodesics, and map normal convex neighborhoods to normal convex neighborhoods.
    As for any diffeomorphism, we have
    \begin{equation} \label{eq:inv_dphi}
        (\dd_p \Phi)^{-1} = \dd_{\Phi(p)} \Phi^{-1} 
    \end{equation}
    for each $p \in M$.
    Further, the exponential and logarithm commute with an isometry $\Phi$ as follows:
    \begin{align}
        \exp_{\Phi(p)}\big(\dd_p\Phi(X)\big) &= \Phi \big( \exp_p(X) \big), \quad X \in T_pM, \label{eq:exp_isom} \\
        \dd_p \Phi \big(\log_p(q) \big) &= \log_{\Phi(p)} \big(\Phi(q)\big), \; q\in D_pM. \label{eq:log_isom}
    \end{align}
    The set of all isometries of $M$ constitutes the isometry group of $M$.

    Another fundamental property of Riemannian geometry is curvature. The geodesics define the real-valued sectional curvatures $K_p$ that intuitively measure how much surface images of planes under $\exp_p$ bend.
    Important model manifolds of constant sectional curvature $K := K_p \equiv c$ are spheres ($c>0$), Euclidean spaces ($c=0$), and hyperbolic spaces ($c<0$).
    The curvature has a profound influence on the properties of the manifold.
    In particular, if $M$ is a so-called Hadamard manifold~\citep[Ch.\ 6]{Petersen2006}, a simply connected manifold with non-positive sectional curvature $K_p \le 0$ everywhere, the whole manifold is a normal convex neighborhood. 
    Apart from Euclidean and hyperbolic spaces, the spaces of symmetric positive definite matrices with the affine invariant metric belong to this class.
		
%----------------------------------------------------------------------        
\subsection{Graph Laplacian in Manifolds}
    Let $G = (V, E, w)$ be a directed graph with vertex (node) set $V = \{v_1, \dots, v_n\}$, directed edge set $E \subset V \times V$, and positive edge weights $w : E \to \mathbb{R}_+$. 
    The ordering of the vertices encodes the direction of the edge, that is, $(v, u)$ is an edge going from $v$ to $u$; we also write $u \sim v$ to denote that there exists such an edge.

    Starting with the set of functions
    $$\mathcal{F}(V, M) := \{f : V \to M \},$$ 
    we define the set of \textit{admissible} $M$-valued vertex functions $\mathcal{H}(V, M) \subseteq \mathcal{F}(V, M)$ by
    $$
             \mathcal{H}(V, M) := \{f\ |\  
            f(u) \in D_{f(v)} M \; \forall (v, u) \in E \}.
    $$
    The definition ensures that $\log_{f(v)}f(u)$ exists whenever $(v, u) \in E$. The equality $\mathcal{H}(V, M) = \mathcal{F}(V,M)$ holds if and only if $\log_p(q)$ is defined for any $p,q \in M$. This is the case, for instance, for Hadamard manifolds. When a graph $G$ comes with a map $f \in \mathcal{H}(V, M)$, we call $f$ \textit{(vertex) features} of $G$; in this case, we also write $G = (V, E, w, f)$.
    
    For $f \in \mathcal{H}(V, M)$, we denote the disjoint union of tangent spaces at the values of $f$ by 
    $$T_f M := \bigcup_{v \in V} T_{f(v)}M$$
    and the space of tangent space functions corresponding to $f$ by
    $$
        \mathcal{H}(V, T_fM) := \{F : V \to T_fM\ |\ F(v) \in T_{f(v)}M \}.
    $$
    Now the graph Laplacian  $\Delta: \mathcal{H}(V, M) \to \mathcal{H}(V, T_fM)$ is defined, for $f \in \mathcal{H}(V, M)$, by
        \begin{equation}\label{eq:graph_laplace}
            \Delta f(v) = -\!\!\!\sum_{u \sim v} w(v,u) \log_{f(v)}f(u) \in T_{f(v)}M;
        \end{equation} 
    see~\cite{bergmann2018graph}.				
    If $M = \mathbb{R}^d$, this reduces to a well-known 
    graph Laplacian for Euclidean functions; see, for example,~\cite{GO2008}.

%----------------------------------------------------------------------------------------
 \subsection{Equivariance and Invariance}
    Equivariance and invariance are essential properties of many neural networks. Let $S$ be a set and $\mathcal G$ a group with neutral element $e$.
    Assume that $\mathcal G$ acts on $S$ from the left; that is, there is a (left) group action $\rho: \mathcal G \times S \to S$ 
    such that $\rho(e, s) = s$ and $\rho(g, \rho(h, s)) = \rho(gh, s)$. 
    A function $\phi: S \to \mathbb{R}$ is called invariant with respect to the group action of $\mathcal G$ if 
    $$\phi(\rho(g,s)) = \phi(s)$$
    for all $g \in \mathcal G$ and $s \in S$.
    A function $\Phi : S \to S$ is called equivariant with respect to the group action of $\mathcal G$ if
    $$\Phi(\rho(g,s)) = \rho(g,\Phi(s))$$
    for all $g \in \mathcal G$ and $s \in S$.
    Note that the concatenation of an equivariant and invariant function is invariant.

    The equivariance of neural networks under a group operation is usually achieved by stacking (that is, concatenating) layers that are themselves equivariant. When an invariant network is the goal, several equivariant layers are followed by at least one invariant layer before the final output is generated.
		
 %-------------------------------------------------------
    \section{Diffusion Layer} \label{sec:diff_layer}
 %------------------------------------------------------

    \begin{figure}[!t]
        \centering
        \includegraphics[width=0.3\textwidth]{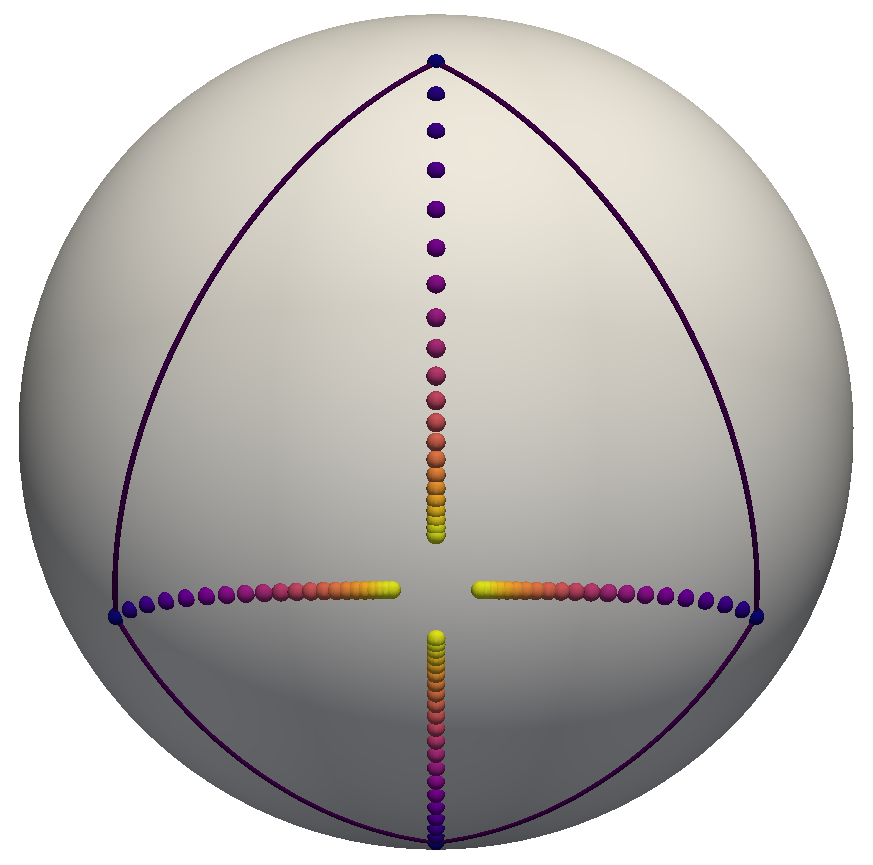}
        \caption{Diffusion (indicated by color) of the vertices of a rectangle graph on the 2-sphere $\textnormal{S}^2$}
        \label{fig:diffusion}
    \end{figure}
 
    In this section, we introduce a novel diffusion layer.
    Let $G=(V, E,w,f)$ be a graph with feature map $f$.
    We augment $f$ with a time parameter $t$ and  consider, for $a>0$, the $M$-valued graph diffusion equation
    \begin{equation} \label{eq:diffusion}
        \begin{cases}
            \frac{\partial}{\partial t} \widetilde{f}(v, t) = -\Delta \widetilde{f}(v, t), & v \in V, \quad t \in (0, a), \\
            \widetilde{f}(v, 0) = f(v), & v \in V.
        \end{cases}
    \end{equation}

    Figure~\ref{fig:diffusion} depicts the diffusion of a rectangle graph with constant edge weights. In Appendix~\ref{app:diff_equation}, we show that there always exists a solution to Equation~\eqref{eq:diffusion} that lives for some time.
    We also have the following theorem.				
%-----------------------------------------------------
    \begin{theorem}\label{thm:sln}
        Let $G = (V, E, w, f)$ be a graph with positive weights and features $f \in \Hsp$ 
		such that the smallest closed geodesic ball that contains the graph's features is convex. 
		Assume $\sum_{u \sim v} w(v, u) \le 1$ for all $v \in V$. 
		Then, Equation~(\ref{eq:diffusion}) has a solution $\widetilde{f} : V \times [0, \infty) \to M $ that is defined for all $t \ge 0$.         
    \end{theorem}				
%-----------------------------------------------------				
The proof is given in Appendix~\ref{app:diff_equation}.
        
As discussed in Section~\ref{sec:related_work}, discrete approximation schemes were used to build neural network layers from diffusion processes in Euclidean space. 
We will transfer this to manifolds. 
To this end, we now introduce the maps that are needed for the 
explicit Euler discretization of Equation~(\ref{eq:diffusion}).

As always in deep learning, we need a nonlinear activation function that prevents depth collapse. To build equivariant/invariant networks, it must not be affected by the isometries of $M$. Since we are not aware of an isometry-equivariant manifold-valued activation function in the literature, we propose the following new one:
With the sigmoid function $\textnormal{sig}(x) := e^x/(1+e^x)$ and $\theta = (\vartheta_1, \vartheta_2) \in \mathbb{R}^2_{\ge 0}$, we define, for any $p \in M$, the nonlinear activation function
$\sigma_p^\theta: T_p M \to T_p M$ by
        \begin{equation*}
               \sigma_p^\theta(X) := \textnormal{sig} \big( \vartheta_1 \|X\|_p - \vartheta_2 \big) X.
        \end{equation*} 
Note that $\sigma_p^\theta$ acts on tangent vectors, which are the images of the Laplace operator~\eqref{eq:graph_laplace}. The activation function has two learnable parameters $\vartheta_1$ and $\vartheta_2$, which increase the network capacity and realize an isometry-equivariant, nonlinear re-scaling of the input vector.

With this activation function, we define for $f \in \mathcal{H}(V,M)$ and $t \ge 0$, 
the $1$-\textit{step map} $\step_{t,\theta} : \mathcal{H}(V,M) \to \mathcal{F}(V, M)$ by
            \begin{equation*}
                \step_{t,\theta} (f) (v) := \exp_{f(v)} \Big(\!\!-t \sigma_{f(v)}^\theta {\left(\Delta f (v) \right)}\Big),\ \ v \in V.
            \end{equation*}
        The map $\step_{t, \theta}$ realizes a ``nonlinearly activated'' explicit diffusion step of length $t$ of Equation~\eqref{eq:diffusion}. 
        Each vector $\Delta f (v)$ can thereby be interpreted as a ``force'' that pushes the feature $f(v)$ towards the weighted Fréchet mean of its graph neighbors; the exponential carries out this movement.
        Note that since the graph Laplacian is only influenced by a node's direct neighbors---its so-called \textit{1-hop neighborhood}---, the 1-step map only aggregates information over these neighborhoods.

\begin{remark} \label{rem:cutLocus}
    As mentioned, $\Hsp = \Fsp$ holds for spaces like Hadamard manifolds. 
    In this case, $\step_{t, \theta}(f)$ must be in the space of admissible functions.
    For spaces with somewhere-positive sectional curvatures,
    $\step_{t,\theta}(f) \notin \mathcal{H}(V, M)$ is generally possible when the features of two vertices that share an edge lie in their respective cut loci. 
    However, any maximal neighborhood $D_p$, on which the exponential is invertible, 
    is dense in $M$; see, for example,~\cite[Cor.\ 28.2]{Postnikov2013}. 
    Thus, every cut locus has a null measure under the manifold's volume measure~\citep[Lem.\ III.4.4]{sakai1996riemannian}. 
    Consequently, also in positively curved manifolds, $\step_{t,\theta}(f)$ 
    lies in $\Hsp$ except for some very exceptional cases.
    
    In Appendix~\ref{app:l_step}, we additionally show that the 1-step map always yields admissible features
    even in positively curved spaces whenever the data is ``local enough'' and the step size is not too large.

    We never experienced problems with the cut locus in any of our experiments presented later.
\end{remark}
   
Since numerical approximation schemes usually employ several approximation steps, we extend the 1-step map for $\ell > 1$ and 
$f \in  \Hsp$ with $\step_{t,\theta}^{l-1}(f) \in \Hsp$ as
            \begin{align} \label{l-step}
                \step^{\ell}_{t,\theta} (f) 
								:= \underbrace{\step_{t,\theta} \circ \dots \circ \step_{t,\theta}}_{\ell \text{ times}}(f).
            \end{align}
Again, $\step^\ell_{t,\theta}$ realizes $\ell$ ``nonlinearly activated'' explicit Euler steps of length $t>0$ 
for Equation~\eqref{eq:diffusion}. 
The larger $\ell$ is, the more the features will be pushed towards a configuration in which each is the weighted Fréchet mean of all its neighbors.
It directly follows that the $\ell$-step map aggregates information over $\ell$-hop neighborhoods. Therefore, a node's new feature is influenced by all nodes connected to it through a sequence of not more than $\ell$ edges. 

Now, we can define our diffusion layer. Let $\Omega$ be the set of graphs with vertex features.
        %-----------------------------------------------------------------------
				\begin{definition}[diffusion layer]\label{def:diffusion_layer}
				For $c \in \mathbb N$, let $\zb t = (t_1,\dots,t_c) \in \mathbb R_{\ge 0}^c$ and
				$\boldsymbol{\theta} = (\theta_1, \dots, \theta_c)\in (\mathbb R_{\ge 0}^2)^c$.
				A diffusion layer with $c \ge 1$ channels is a map $\diff_{\zb t,\boldsymbol{\theta}}: \Omega^c \to \Omega^c$ with
            \begin{align*} \label{eq:diffusion_layer}
               \diff_{\zb t,\boldsymbol{\theta}} &\big((G_i)_{i=1}^c \big):= \Big( \big(V_i,E_i,w_i, 
							\step^\ell_{t_i, \theta_i}(f_i) \big) \Big)_{i=1}^c.
            \end{align*}
        \end{definition}
				%-----------------------------------------------------------------------
        The diffusion layer takes $c$ graphs as input and diffuses them for certain amounts of time. Thus, it consists of $c$ diffusion ``channels''.
        The learnable parameters are the diffusion times $\zb t$ and the step activations $\boldsymbol{\theta}$.
        An essential property of the new layer is that it can be applied to graphs with varying numbers of nodes and connectivity patterns. 
        It is also clear that the information aggregation takes place in $\ell$-hop neighborhoods. 
        This local aggregation of information is considered one of the keys to the success of convolutional (graph) neural networks in many situations~\citep{bronstein2021geometric, kipf2017semisupervised}.

        Apart from the qualities already discussed, the diffusion layer possesses the following equivariance properties:
%------------------------------------------------------------------------
        \begin{theorem}
            The diffusion layer is equivariant under node permutations and isometric transformations of the feature manifold.
        \end{theorem}
%-----------------------------------------------------------------------				
        \begin{proof}
           We can restrict ourselves to one channel. The equivariance of the diffusion layer under permutations of the vertex set is already inherent in our formulation as $f$ and $\diff$ are defined on the \textit{unordered} set $V$: If an order is imposed on the vertex set, the features are sorted accordingly.

            Next, we deal with the equivariance under isometries. Let $\Phi: M \to M$ be an isometry. We must show that transforming features with $\Phi$ commutes with applying the diffusion layer: 
            $$\step^\ell_{t, \theta}(\Phi \circ f) = \Phi \circ \step^\ell_{t, \theta}(f).$$
		    It suffices to show the relation for $\ell = 1$.
            By definition, we have, for all $v \in V$,
            \begin{equation*}
                \step_{t, \theta}(\Phi \circ f)(v) = \exp_{\Phi(f(v))}(-X_v)
            \end{equation*}
            with
            \begin{align*}
                X_v &:= t\sigma^{\theta}_{\Phi(f(v))} \Big(\!-\sum_{u \sim v} w(v,u) \log_{\Phi(f(v))}\Phi(f(u)) \Big).
			\end{align*}	
			Using \eqref{eq:log_isom}, the facts that $\dd_{f(v)}\Phi$ is linear and norm-preserving, and \eqref{eq:graph_laplace} gives 
            \begin{align*}
                X_v &= t\sigma^{\theta}_{\Phi(f(v))} \Big(\!-\sum_{u \sim v} w(v,u) (\dd_{f(v)} \Phi)(\log_{f(v)}f(u) )\Big) \\
                &= \dd_{f(v)} \Phi \bigg(t \sigma_{f(v)}^{\theta} \Big(\!-\sum_{u \sim v} w(v,u) \log_{f(v)}f(u) \Big) \bigg) \\
                &= \dd_{f(v)} \Phi \Big(t \sigma_{f(v)}^{\theta} \big(\!\Delta f(v) \big) \Big).
            \end{align*}
			Thus, applying \eqref{eq:exp_isom} and then \eqref{eq:inv_dphi} yields
			\begin{align*}
                \step_{t, \theta}(\Phi \circ f)(v) &= \exp_{\Phi(f(v))}(-X_v) \\
                &= \Phi \bigg(\!\exp_{f(v)} \Big(\dd_{\Phi(f(v))} \Phi^{-1} (-X_v) \Big) \bigg)\\
                &= \Phi \bigg(\!\exp_{f(v)} \Big(-t \sigma_{f(v)}^{\theta} \big(\!\Delta f(v) \big) \Big) \bigg) \\
								&= \Phi \circ \step_{t, \theta}(f) (v).
            \end{align*}
            This proves the claim.
        \end{proof}
        Our diffusion layer is equivariant under the symmetries of the graph and the feature space. As already discussed earlier, using only layers and networks that exhibit this behavior by default is a highly successful ansatz that helps to counter the curse of dimensionality~\citep{bronstein2021geometric}. Thus, the diffusion layer is a highly versatile building block for GNNs.

%------------------------------------------------------------------
\section{Tangent Multilayer Perceptron} \label{sec:tLL}
%------------------------------------------------------------------
        This section introduces a ``1x1 convolution'' layer that can transform node features and alter the network's width.
        To this end, we transfer ideas from the vector neuron framework~\citep{deng2021vector} to the manifold setting to construct a linear layer on tangent vectors with a subsequent nonlinearity. Due to the latter, the layer can be stacked to increase the depth. Indeed, it works like the well-known multilayer perceptron (MLP).
        
        The layer transforms $c_\textnormal{in}$ features $f_1, \dots, f_{c_\textnormal{in}} \in \Hsp$ from $c_\textnormal{in}$ channels into $c_\textnormal{out}$ output feature channels $g_1, \dots, g_{c_\textnormal{out}} \in \Hsp$. 
        To this end, the node features are mapped to a reference tangent space, and two sets of linear combinations are learned. One set of vectors is needed to define positive and negative half-spaces in the tangent space. The other---output---set is then (nonlinearly) transformed according to the vectors' half-spaces. 
        Finally, the exponential map is applied to transfer the result back to the manifold. 
        %----------------------------------------------------------------------
	\begin{definition}\label{def:tangent_mlp} (Tangent MLP)
            For $f_1,\dots,f_{c_\textnormal{in}} \in \Hsp$ and each $v \in V$, let $\bar f(v) \in M$ be a point depending on $\zb f_v := \left(f_i(v) \right)_{i=1}^{c_\textnormal{in}}$ that is equivariant under node permutations and isometries of $M$. Further, assume that $f_1(v),\dots,f_{c_\textnormal{in}}(v) \in D_{\bar f (v)}M$. The \textit{tangent perceptron} with weights $\omega^i_j,\xi^i_j$, $i=1,\dots,c_\textnormal{in}$, $j=1,\dots,c_\textnormal{out}$ and nonlinear scalar activation $\sigma :\mathbb{R} \to \mathbb{R}$ transforms the features of each node into new features $g_1,\dots,g_{c_\textnormal{out}} \in \Hsp$ in three steps: 
		First, the directions
            \begin{align*}
                X_j(\zb f_v) &:= \sum_{i=1}^{c_\textnormal{in}} \omega_j^i\log_{\bar f(v)} \big(f_i(v) \big), \\
								\widetilde Y_j(\zb f_v) &:= \sum_{i=1}^{c_\textnormal{in}} \xi_j^i \log_{\bar f(v)} \big(f_i(v) \big),\\
								Y_j(\zb f_v) &:= \frac{\widetilde Y_j(\zb f_v)}{\|\widetilde Y_j(\zb f_v)\|_{\bar f(v)} } 
            \end{align*}
            are computed, and the $X_j$ are orthogonally decomposed with respect to the $Y_j$ as
            \begin{align*}
			X_j^{\tang}(\zb f_v) &:= \big\langle X_j(\zb f_v), Y_j(\zb f_v) \big\rangle_{\bar f(v)} Y_j(\zb f_v),\\
                X_j^{\perp}(\zb f_v) &:= X_j(\zb f_v)- X_j^{\tang}(\zb f_v). 
            \end{align*}
            Second, the nonlinear scalar function $\sigma$ is applied
            \begin{align*}
                Z_j(\zb f_v) 
								&:= 
								\frac{\sigma\big(\langle Y_j(\zb f_v), X_j^{\tang}(\zb f_v) \rangle_{\bar f(v)} \big)}{\langle Y_j(\zb f_v), X_j^{\tang}(\zb f_v) \rangle_{\bar f(v)}} X_j^{\tang}(\zb f_v) 
								+ X_j^{\perp}(\zb f_v),
						\end{align*}	
							and third, the result is transformed back to $M$ via
							\begin{align*}
                g_j(v; \zb f_v) &:= \exp_{\bar f(v)} \big(Z_j(\zb f_v) \big).
            \end{align*}
            The \textit{tangent multilayer perceptron (tMLP) with $m$ layers} is the concatenation of $m$ tangent perceptrons. 
        \end{definition}
%------------------------------------------------------------------------------					
Figure~\ref{fig:tMLP} depicts how the tMLP works when $\sigma \equiv \textnormal{ReLU}$. In this case, vectors are projected onto the positive half-space. The subscripts and arguments are omitted to avoid clutter. 

We give some further remarks on the layer.
%------------------------------------------------------------------------------				
        \begin{remark}
             \begin{itemize}
                \item[i)] Different choices are possible for the reference point $\bar f$. For example, we can select one of the channel features, that is, $\bar{f}(v):= f_i(v)$ for some $i \in \{1,\dots, c_{\textnormal{in}} \}$, or the Fréchet mean~\citep{Pennec2006} of some of the $\left( f_i(v)\right)_{i=1}^{c_{\textnormal{in}}}$; the result of a local or global pooling operation can also be used. Because it does not require additional computations, we select one of the channel features.
			    \item[ii)] Classical options for the nonlinear function $\sigma$ are the rectified linear unit (ReLU) or leaky ReLU. 
                \item[iii)] For a tMLP with $m > 1$, using the same reference point in all layers is convenient. With this choice, the computation of $g_j(v; \zb f_v)$ and the applications of the logarithm in the next layer (for $X_j(\zb f_v)$ and $Y_j(\zb f_v)$) cancel each other; hence, these operations can be left away for higher computational efficiency. 
                \item[iv)] The arguments from Remark~\ref{rem:cutLocus} concerning the cut locus of a non-Hadamard manifold also apply to the tMLP.
                \item[v)] Instead of the tMLP, we could also use sequences of 1D-convolutions from~\cite{chakraborty2020manifoldnet}. However, when stacked, these require nonlinear activations to guarantee non-collapsibility, and, as mentioned before, we are not aware of isometry-equivariant manifold-valued activation functions in the literature.
            \end{itemize}
        \end{remark}
        \begin{figure}
            \centering
            \includegraphics[width=0.8\linewidth]{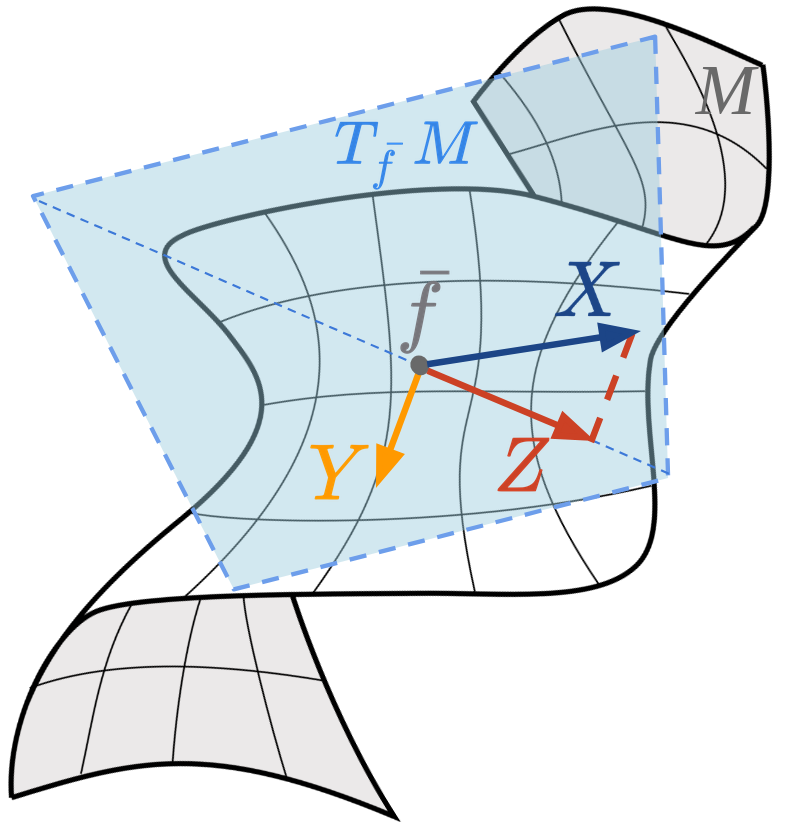}
            \caption{tMLP with ReLU activation for a fixed node and output channel}
            \label{fig:tMLP}
        \end{figure}
        
%------------------------------------------------------------------------------	        
        The following theorem shows that the tMLP shares the equivariance properties of the diffusion layer.
%------------------------------------------------------------------------------	
        \begin{theorem}
            The tMLP is equivariant under node permutations and isometric transformations of the feature manifold.
            \end{theorem}
%------------------------------------------------------------------------------					
        \begin{proof}
            We show the claim for a tangent perceptron. This directly yields the result for the general tMLP since the latter concatenates several of the former.
            
            The equivariance under permutations of the nodes holds because the weights $\omega^i_j, \xi^i_j$ have the same value for every node, and $\overline{f}$ is permuted just like the underlying vertices.
            
            Let $\Phi: M \to M$ be an isometry. 
						By~\eqref{eq:log_isom}, linearity of $\dd_{\bar f(v)} \Phi$, and since the reference points $\bar f$ transforms with the isometry, 
						we obtain with $\Phi \circ \zb f_v := \big( \Phi\left(f_i(v) \right) \big)_{i=1}^{c_\textnormal{in}}$ that
            \begin{align*}
                X_j (\Phi \circ \zb f_v) 
								&=\sum_{i=1}^{c_\textnormal{in}} \omega_j^i\log_{\Phi(\bar f(v))} \Big( \Phi \big(f_i(v) \big) \Big) \\
                &= \dd_{\bar f(v)} \Phi \left(\sum_{i=1}^{c_\textnormal{in}} \omega_j^i\log_{\bar f(v)} \big(f_i(v) \big) \right) \\
                &= \dd_{\bar f(v)} \Phi \big( X_j (\zb f_v)  \big).
            \end{align*}
            Analogously, since $\dd_{\bar f(v)} \Phi$ preserves norms, we obtain
            \begin{equation*}
                \widetilde Y_j (\Phi \circ \zb f_v)  
								= \dd_{\bar f(v)} \Phi \big( \widetilde Y_j  (\zb f_v) \big),
            \end{equation*}
						and further
						\begin{align*}
                 Y_j (\Phi \circ \zb f_v)  
								&= \frac{\dd_{\bar f(v)} \Phi \big( \widetilde Y_j  (\zb f_v) \big)}
								{\| \dd_{\bar f(v)} \Phi \big( \widetilde Y_j  (\zb f_v) \big)\|_{\Phi(\bar f(v)) }} \\
								&= \dd_{\bar f(v)} \Phi \big(Y_j (\zb f_v) \big).								
            \end{align*}
            Then, since isometries preserve angles as well, we get
						\begin{align*}
						X_j^{\tang}&(\Phi \circ \zb f_v) \\
														&= \big\langle X_j(\Phi \circ\zb f_v), 
								Y_j(\Phi \circ\zb f_v)
								\big\rangle_{\Phi(\bar f(v))} Y_j(\Phi \circ \zb f_v)\\
								&= \dd_{\bar f(v)} \Phi \big( X_j^{\tang}(\zb f_v) \big),
								\\
                X_j^{\perp}&(\Phi \circ \zb f_v) \\ 
                &=  \dd_{\bar f(v)} \Phi \big( X_j(\zb f_v)- X_j^{\tang}(\zb f_v) \big) 
								\\
								&= \dd_{\bar f(v)} \Phi \big( X_j^{\perp}(\zb f_v) \big).
													\end{align*}
            Thus, we conclude
                \begin{align*}
                Z_j &\big( \Phi \circ \zb f_v \big) \\
								&=
								\frac{\sigma \big( \langle Y_j(\Phi \circ \zb f_v), X_j^{\tang}(\Phi \circ \zb f_v) \rangle_{\Phi(\bar f(v))} \big)}
								{\langle Y_j(\Phi \circ \zb f_v), X_j^{\tang}(\Phi \circ \zb f_v) \rangle_{\Phi(\bar f(v))}} X_j^{\tang}(\Phi \circ \zb f_v) 
								\\
								& \quad + X_j^{\perp}(\Phi \circ \zb f_v)\\
								&= \dd_{\bar f(v)}\Phi \big(Z_j( \zb f_v ) \big).
            \end{align*}
            Finally, \eqref{eq:exp_isom} yields
            \begin{align*}
                g_j \big(v; \Phi \circ \zb f_v \big)  
								&= \exp_{\Phi(\bar f(v))} \Big(Z_j (\Phi \circ \zb f_v ) \Big)  \\
                &= \exp_{\Phi(\bar f(v))}  \Big(\dd_{\bar f(v)} \Phi \big( Z_j (\zb f_v \big) \Big) \\
                &= \Phi \Big( \exp_{\bar f(v)} \big( Z_j (\zb f_v) \big) \Big) \\
                &= \Phi \big( g_j (v; \zb f_v ) \big),
            \end{align*}
              which finishes the proof.
        \end{proof}
        Note that the tMLP is \textit{not} equivariant under permutations of the channels. Unlike the ordering of the nodes, the model learns the ordering of the channels, and we \textit{want} to give the network the freedom to extract information from it.
%--------------------------------------------------------
    \section{Manifold GCN} \label{sec:GCN}
%--------------------------------------------------------	
 We now describe a generic graph convolutional neural network (GCN) block to which only the last task-specific layers need to be added for a complete architecture: 
        It consists of a sequence
        $$\diff \rightarrow \textnormal{tMLP} \rightarrow \dots \rightarrow \diff \rightarrow \textnormal{tMLP},$$
        in which the number of channels of each layer can be chosen freely, as long as the number of output channels of one layer equals the number of input channels of the next. When a graph is fed into the network, $c$ copies of it are given to the first diffusion layer, where $c$ is the number of channels it has. While the diffusion layers aggregate information over the graph, the tMLPs ensure that the architecture is not equivalent to a single diffusion layer (that is, it does not collapse). The tMLPs increase the nonlinearity and allow the network to transfer information between channels.

        \begin{figure*}
            \centering
            \includegraphics[width=\linewidth]{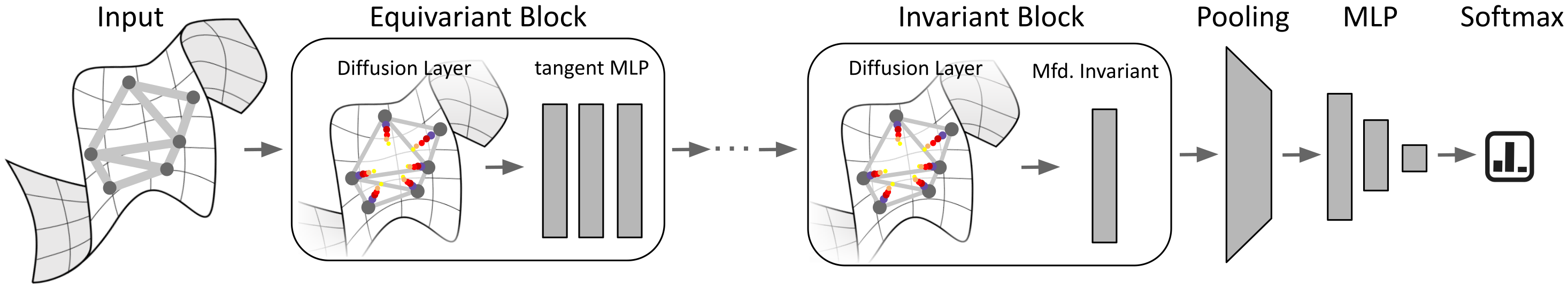}
            \caption{Manifold GCN pipeline for graph classification}
            \label{fig:architecture}
        \end{figure*}

        Later, when we speak of the ``depth'' of a manifold GCN block, the number of diffusion layers is meant. When all layers have the same number of channels $c$, we say that the block has ``width'' $c$. 

        When learning scalar outputs (as, for example, in classification), it is usually beneficial to add layers after the manifold GCN block that make the network invariant under node permutations and isometries. We found that the manifold invariant layer (Mfd. invariant) from~\citep[Sec.\ 2.3]{chakraborty2020manifoldnet} works well; we apply it node-wise and add graph-wise (max- and/or mean-)pooling afterward.
        While the invariant layer makes the network invariant under isometries~\citep[Prop.\ 3]{chakraborty2020manifoldnet}, the global pooling operation ensures invariance under node permutations.
        An MLP with a trailing softmax function can finally be added to map, for example, to class probabilities for classification. Figure~\ref{fig:architecture} shows the proposed classification pipeline.

        We found that normalizing the graph weights can boost the performance of GNNs that use a manifold GCN block. The reason, in all likelihood, is that if the step size is too large, the $\ell$-step map will be far from the continuous flow so that the evolution of the graphs becomes ``chaotic''. Before feeding a graph $G = (V, E, w, f)$ with $b:= \max_{v \in V} \sum_{u \sim v} w(v, u) > 1$ into the network, we thus recommend a global scaling of the weights: 
        $$\widetilde{w} := \frac{w}{b}.$$
        We have observed the best performance with ``short'' initial diffusion times (usually distributed tightly around some $t_0 \in [0, 1]$).
        Appendix~\ref{app:l_step}, which discusses how $\diff$ behaves in specific scenarios, further explains why normalization is helpful.

        \textbf{Complexity analysis:} Considering the number of steps $\ell$ a fixed parameter, evaluating Equation~\eqref{l-step} has complexity $\mathcal{O}(|E|)$. Since the node-wise operations of a tMLP are of class $\mathcal{O}(|V|)$, a manifold GCN block also has complexity $\mathcal{O}(|E|)$, like the standard message-passing framework.
        
        The costs include constants that depend on the complexity of the Riemannian operations (most importantly, the exponential and logarithm maps). For some applications (compare Section~\ref{exp:synthetic}), it is interesting how their complexity depends on the manifold's dimension: This differs from manifold to manifold; for example, in hyperbolic and spherical spaces, there is only a linear dependency, whereas they scale cubically for the SPD space.

        %\textcolor{red}{For some applications (for example, learning embeddings for abstract graphs), it is also interesting to consider the cost for different dimensions of embedding manifold (cf. Section~\ref{sec:gnns_manifolds})}

\section{Experiments} \label{sec:numerics}
    The results of our experiments are in this section. We used our manifold GCN in two graph classification tasks.

    \subsection{Synthetic Graphs} \label{exp:synthetic}
        A benchmark~\citep{liu2019hyperbolic} for graph classification algorithms is to let them learn whether a graph was created using the Erdős-Renyi~\citep{erdHos1960evolution}, Barabasi-Albert~\citep{barabasi1999emergence}, or Watts-Strogatz~\citep{watts1998collective} algorithm. 
        Each algorithm creates graphs with different characteristics: Erdős-Rényi graphs are purely random with low clustering and no hubs, Barabási–Albert graphs grow with preferential attachment and naturally form hubs with a power-law degree distribution, while Watts–Strogatz graphs create highly clustered small-world networks by rewiring a regular lattice.
        Examples of graphs created with each of them are shown in Figure~\ref{fig:graphs}.
        
        \cite{liu2019hyperbolic} and \cite{dai2021hyperbolic} showed that classifying the algorithm from a learned embedding in hyperbolic space is superior to classifying from a Euclidean one.
        We used our manifold GCN for this task. 
        Since our network can work on arbitrary manifolds, we also compare the hyperbolic space to the SPD space as an embedding space. The latter is thought to be advantageous because it has a more complex geometric structure~\citep{zhao2023modeling}. 

        \begin{figure*}[!ht]
            \centering
            \includegraphics{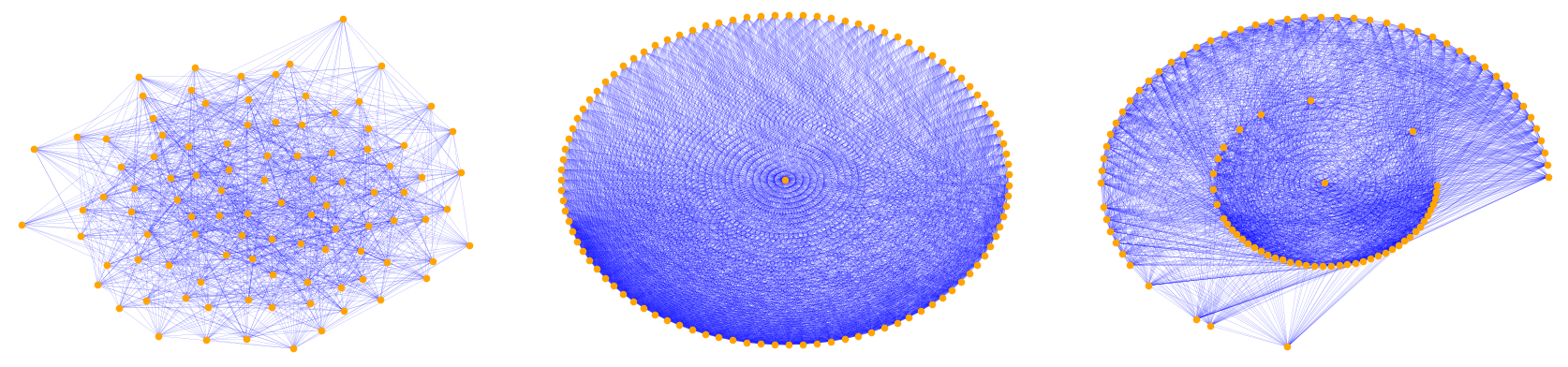}
            \caption{Graphs created with the Erdős-Renyi (left), Barabasi-Albert (middle), and Watts-Strogatz (right) algorithms}
            \label{fig:graphs}
        \end{figure*}

        \subsubsection{Data}
            We used differently-sized data sets of graphs built with the three creation algorithms. Each set was balanced; that is, it contained the same number of graphs for each algorithm. 
            Every graph had 100 nodes. The number of edges was chosen (as by~\cite{liu2019hyperbolic}) in the following way: For Barabási-Albert graphs, the number of edges to attach from a new node to existing nodes was discrete-uniformly distributed between 1 and 200. For Erdős-Rényi, the probability of edge creation was chosen from a uniform distribution on $[0.1,1]$. 
            For Watts-Strogatz, each node was connected to a discrete-uniform distributed random number between 1 and 200 neighbors in the ring topology, and the probability of rewiring each edge was taken from a uniform distribution on $[0.1,1]$. 
            All edge weights were set to 1/100.
            
        \subsubsection{Architecture and Loss}
            We used a manifold GCN block of depth 2. The initial flow layer had 5 channels; the following tMLP, which consisted of only one layer and used leaky ReLU as nonlinearity, increased this number to 16. Another flow layer followed, and we added the invariant layer from~\cite{chakraborty2020manifoldnet} to map to scalar features. The invariant layer calculated the distances to two weighted means for each node, outputting scalar features in 16 channels. Graph-wise max- and mean-pooling were then used for aggregation; a 2-layered MLP with leaky ReLU activations, and a final softmax function mapped to the class probabilities.
            The network was trained with the standard cross-entropy loss.

        \subsubsection{Embedding Manifolds} 
            In our experiments, we embedded the graphs in hyperbolic and SPD spaces.
            As the model of hyperbolic space, we used the Lorentz model (see, for example,~\cite{liu2019hyperbolic} for formulas), as it performed best on the given task~\citep[Table 1]{liu2019hyperbolic}. When we used the SPD space, we endowed the latter with the affine-invariant metric~\citep{Pennec2006}.
            %; this turns it into a negatively curved Hadamard manifold. 
            
        \subsubsection{Initial Encoding}
            For the initial embedding of a graph into hyperbolic space, we used the following options: 
            \begin{itemize}
                \item[(i)] The ``degree'' embedding from~\cite{liu2019hyperbolic}: We assign to each node a one-hot vector $X$ of length 101, the index of the one-entry being the degree of the node. Since the last entry of $X$ is 0, this vector is a tangent vector at the origin $o = [0\ \cdots\ 0\ 1]^T$ of the 100-dimensional Lorentz-hyperboloid $\textnormal{H}^{100}$. 
                We can then use a (learned) linear function that maps to the tangent space at the origin $o_d$ of a $d$-dimensional Lorentz space. 
                The final step is then to apply the exponential $\exp_{o_d}$ of the Lorentz model to obtain a feature $\exp_{o_d}(X) \in \textnormal{H}^d$ for the node.
                \item[(ii)] The ``one-hot'' embedding: We choose an order for the graph's nodes and assign to each node the one-hot vector $X \in T_o\textnormal{H}^{100}$ of length 101 that indicates its place. The final feature is $\exp_o(X)$.
            \end{itemize}
            While the first embedding is equivariant under node permutations because of the use of the degree, the one-hot embedding is \textit{not}. Therefore, networks such as~\cite{liu2019hyperbolic} and \cite{dai2021hyperbolic} cannot use it without sacrificing invariance under permutations. For our manifold GCN, this is different since two one-hot embeddings can be transformed into each other by a series of rotations (about the hyperboloid's symmetry axis) and reflections (about hyperplanes defined by the coordinate axes). Since the latter are isometries of the Lorentz model, the equivariance under isometries of the proposed layers secures the permutation-invariance of the network, even with the one-hot embedding.
            % This argument no longer holds when $X$ is mapped linearly to another tangent space, so the embedding must be in $\textnormal{H}^{100}$.

            We also used a ``one-hot'' embedding similar to the hyperbolic space for the SPD space. There, the tangent space at the identity matrix consists of all symmetric matrices of the same size. We embedded the graphs in the $15$-by-$15$ SPD matrices, as this is the smallest space whose number of free off-diagonal entries is not smaller than the number of nodes. Each node was assigned a different symmetric ``one-hot'' matrix with two one-entries off the diagonal and zero everywhere else. This leads again to a permutation-invariant manifold GCN since we can always transform a one-hot matrix into another via a congruence transformation with the product of two elementary permutation matrices. This represents an isometry since the affine-invariant metric is invariant under congruence with orthogonal matrices~\citep{thanwerdas2023n}.

            \begin{table*}[!ht]
                \centering
                \caption{Classification of synthetic graphs in Hyperbolic space}
                \label{tab1}
                \begin{adjustbox}{width=.9\textwidth}
                \begin{tabular}{ c | c c c c c | c}
                \textbf{Method} & \multicolumn{5}{c|}{\textbf{Mean F1 Score}} & \textbf{\# param.} \\
                \# Graphs & 90 & 180 & 360 & 1080 & 2880 & \\
                \hline
                HGNN & $41.8 \pm 10.3$ & $41.3 \pm 10.5$ & $40.4 \pm 9.7$ & $54.7 \pm 15.4$ & $76.6 \pm 7.9$ & 161903 \\
                H2H-GCN & $42.8 \pm 11.4$ & $54.0 \pm ~9.2$ & $66.2 \pm 6.1$ & $77.6 \pm ~3.7$ & $84.7 \pm 2.5$ & 31348 \\
                Ours (degree) & $55.2 \pm 12.0$ & $63.9 \pm ~8.2$ & $67.2 \pm 5.9$ & $71.3 \pm ~4.7$ & $73.3 \pm 4.0$ & 11970 \\
                Ours (one-hot) & $\textbf{67.8} \pm 12.2$ & $\textbf{73.7} \pm ~8.1$ & $\textbf{76.7} \pm 6.0$ & $\textbf{79.2} \pm ~3.5$ & $\textbf{85.2} \pm 2.0$ & \textbf{1970} \\
                \end{tabular}
                \end{adjustbox}
             \end{table*}

        \subsubsection{Evaluation and Comparison Methods}
            As comparison methods, we used the \textit{hyperbolic graph neural network (HGNN)} from~\citep{liu2019hyperbolic} and the \textit{hyperbolic-to-hyperbolic graph convolutional neural network (H2H-GCN)} from~\citep{dai2021hyperbolic}, which achieved state-of-the-art results for classifying graph construction algorithms.
            
            We tested HGNN and H2H-GCN, embedding the graphs in the Lorentz model of 100-dimensional hyperbolic space.
            We compared them to two versions of our network: one applying the degree and the other the one-hot embedding, both in a 100-dimensional hyperbolic space.
            
            To investigate how well the networks learn with increasing training data, we tested them on data sets with 90, 180, 360, 1800, and 2880 graphs.
            We always split the data set into training, validation, and test sets using a 4:1:1 ratio and used the (macro) F1 score to measure classification accuracy. After each epoch, the accuracy on the validation set was computed. Our final model was chosen as the last one with the highest validation score, and its score on the test set is reported.
            For all but the set with 2880 graphs, we repeated this process 100 times, each time with a new (random) data set; because the standard deviation of the results went down (and the computation times up), we only performed 25 repetitions for the largest set.

            To check the performance on large data sets, we created a single data set of 6000 graphs and trained all methods on two splits; we only employed the better-working one-hot variant of our network here. In addition, we also trained our network using the SPD(15) space to compare embedding manifolds.

            Finally, to assess the sensitivity of our network to the choice of its hyperparameters, we trained on 2880 graphs with a varying depth of the manifold GCN block and a varying number of channels.
            
            %The evaluation method for this experiment was 3-fold cross-validation.
        \subsubsection{Software}
            All our experiments were performed in Python 3.11. For computations in the hyperbolic and SPD spaces, we used \texttt{Morphomatics 4.0}~\citep{Morphomatics}. The graphs were created with \texttt{NetworkX 3.4.2}~\citep{SciPyProceedings_11} and \texttt{Jraph 0.0.6.dev0}. 
            For experiments with HGNN, we used the code that the authors offer online\footnote{\href{https://github.com/facebookresearch/hgnn}{https://github.com/facebookresearch/hgnn}}. Since we could not find the code of H2H-GCN online, we implemented it ourselves in \texttt{Flax 0.9.0} using \texttt{JAX 4.33}. The training was done on a GPU using the ADAM implementation of \texttt{Optax 0.2.3}. The parameters were updated using an incremental update function with step size $0.1$.

        \subsubsection{Parameter Settings}
            We used a learning rate of $10^{-3}$ and trained with balanced batches of size 3 for 60 epochs based on our empirical observations. Since the hyperparameters of H2H-GCN were not made public, we chose 4 layers and 15 centroids for this network based on a grid search. (Refer to~\citep{dai2021hyperbolic} for a description of the hyperparameters.) All computations were performed with double precision.

        \subsubsection{Results and Discussion}
            
            The average F1 scores, standard deviations, and the number of trainable parameters for each network are presented in Table~\ref{tab1}. Manifold GCN with the one-hot embedding performs (often clearly) better than its competitors for each sample size. This is impressive, especially since it is also the method with the fewest trainable parameters. 
            (The significant difference in the number of parameters for our methods is due to the linear layer in the node embedding.) The results indicate that the more substantial separation of the embedded nodes helps the network learn. 
            The reliance of the HGNN and H2H-GCN models on the degree embedding for permutation-invariant classification likely impacts their performance as well.
            
            We observe further that our method performs better with less training data: The other networks' performances only start to come closer to ours when there are 1080 graphs in the training set. We think the better performance on small data sets comes from the inductive bias we introduce through our isometry-invariant network.

            On the large data sets consisting of 6000 graphs, we obtained the following mean classification accuracies:\footnote{The difference between our results and those of~\cite{liu2019hyperbolic} is not surprising since graphs with up to 500 nodes were used there; our results show that the task is harder when there are fewer nodes.}
            \begin{itemize}
                \item Ours in $\textnormal{H}^{100}$: $86.1\%$,
                \item Ours in SPD(15): $85.0\%$,
                \item HGNN: $85.6\%$, 
                \item H2H-GCN: $80.9\%$.
            \end{itemize}
            These results confirm our findings. They further indicate that, for this task, the SPD space is not superior to the hyperbolic space as an embedding space. 

            Our network's results on 2880 graphs with varying depth and width are shown in Table~\ref{tabvar}; they indicate that the network's performance does not strongly depend on a particular choice, as the performance is quite stable.

            Finally, we also compared the runtime of a forward pass to add practical insight to the discussion on complexity in Sec.~\ref{sec:GCN}. Using an NVIDIA GeForce RTX 3080, we obtained the following results: 
            \begin{itemize}
                \item Ours in $\textnormal{H}^{100}$: $0.026$ s,
                \item HGNN: $0.006$ s,
                \item H2H-GCN: $0.050$ s.
            \end{itemize}
            The runtime differences from our model are relatively small ($<5\times$) given that the architectures were selected based on performance rather than capacity.

\begin{table}[ht]
        \centering
        \caption{Results on 2880 graphs with varying depth $d$ and number of channels $c$ of the manifold GCN block}
        \label{tabvar}
        \begin{tabular}{ c | c  c  c}
           & \multicolumn{3}{c}{\textbf{$c$}} \\
                   \textbf{$d$}  & 12 & 16 & 24 \\
            \hline
            1 & $83.6 \pm 1.9$ & $84.2 \pm 2.3$ & $84.9 \pm 2.5$ \\
            2 & $83.6 \pm 2.6$ & $\textbf{85.2} \pm 2.0$ & $84.3 \pm 2.1$ \\
            3 & $82.6 \pm 2.6$ & $82.7 \pm 3.7$ & $82.6 \pm 3.8$ \\
        \end{tabular}
\end{table}

    \subsection{Alzheimer's disease}

        \begin{figure*}[!ht]
            \centering
            \includegraphics[width=0.9\textwidth]{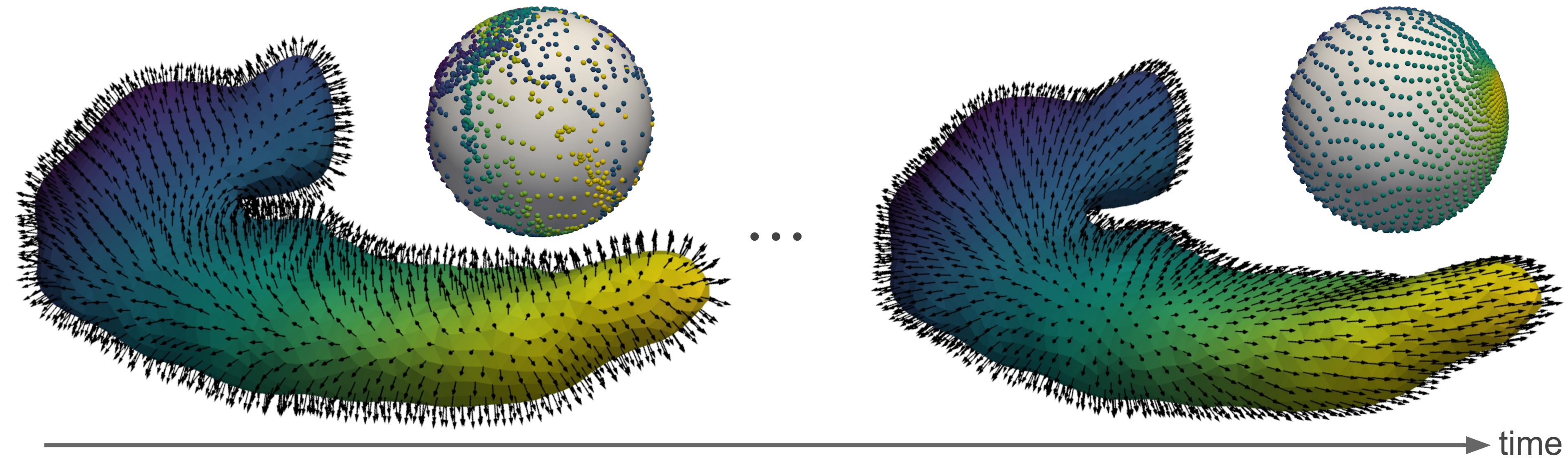}
            \caption{Diffusion of a hippocampus mesh with vertex normals; the same normals, colored like their footpoints, are also shown as points on $\textnormal{S}^2$}
            \label{fig:adni}
        \end{figure*}
        
        Alzheimer's disease (AD) is a neurodegenerative disorder that is diagnosed more and more often worldwide. 
        Studies have shown that AD is characterized by a pattern of brain atrophy, particularly of the hippocampus~\citep {Mueller_ea2010}. 
        
        Given triangle meshes of right hippocampi, we applied two manifold GCN architectures to classify AD from the 1-skeletons\footnote{1-skeletons consist of vertices and edges and represent graphs.} of the surface meshes and their dual graphs. 
        As features, we used gross volumes and two types of manifold-valued shape features.
        
        %There are different representations of the (local) shape of a surface, and many of those take values in manifolds. In this experiment, we compared two of those. 
        
        First, we used the primary graph and assigned surface normals as vertex features. The extrinsic curvature they encode provides a concise summary of the local shape of the hippocampus, adding important information about the condition of the organ.
        Since normals have unit norm, the feature manifold $M$ for this choice is the standard 2-sphere $\textnormal{S}^2$. 
        
        Second, we employed the representation by differential coordinates~\citep{vonTycowicz_ea2018}. These encode hippocampal shapes as deformations of a reference (mean) hippocampus. Applied to triangle meshes, the representation assigns a pair of matrices from $M = \textnormal{SO}(3) \times \textnormal{SPD}(3)$ to each triangle of the mesh; it describes the rotation and stretch, respectively, that the triangle undergoes during the deformation. Using the dual graph, whose vertices correspond to the mesh's triangles, we assigned to each the corresponding pair of matrices. The product manifold $M$ is a relatively complex manifold: It has positive and negative sectional curvature because $\textnormal{SO}(3)$ (with its bi-invariant metric) provides the former~\citep{MILNOR1976293} and SPD(3) (with the affine-invariant metric) the latter.
        
        With both choices, the classification does not depend on the orientation of the hippocampi (meshes) in space. 
        We can thus expect that equivariance under rotational symmetries is beneficial for this task.

        \subsubsection{Data}
            For our experiment, we prepared a data set consisting of 60 subjects diagnosed with AD and 60 cognitively normal (CN) controls based on data from the open-access Alzheimer's Disease Neuroimaging Initiative (ADNI) database.
            Among others, ADNI provides 1632 magnetic resonance images of brains with segmented hippocampi, and triangle meshes thereof were derived by~\cite{AmbellanZachowTycowicz2021}.
            Each mesh was interpreted as a (primal) graph with edge weights being those of the mesh's cotangent Laplacian; constant (normalized) weights were assigned to its dual graph.
            Features for the primal graphs were the vertex normals---the average normal of the faces to which the vertex belonged---and the differential coordinates for the dual graph.
            
            Due to the equivariance of the representations, our network does not need an alignment of the hippocampi. Nevertheless, the imaging protocol assured a coarse alignment in the coordinate system of the MRI scanner. A visualization of the diffusion towards the stable state of a hippocampus mesh with its normals and their location on the sphere is shown in Figure~\ref{fig:adni}.
            
        \subsubsection{Architecture and Loss}
            We trained a manifold GCN on both representations with slightly different setups: Manifold GCN blocks with a depth of 4 and 16 channels in each layer were used in both cases, but for the primal graphs with normals as features there was only one such block, whereas for the dual graphs with differential coordinates as features we used two independently (one for the rotations and one for the SPD matrices). Each tMLP consisted of one layer with a leaky ReLU activation. The last layers were the same as in the experiment with synthetic graphs (an invariant layer, a pooling layer, and a final MLP); the only differences were that each had 16 channels, and the hippocampus volume was appended to the output of the pooling layer.
            
            The networks were also trained using the cross-entropy loss.

        \subsubsection{Evaluation and Comparison Methods}
            We used three different comparison methods. 
            The first one---DiffusionNet~\citep{sharp2022diffusionnet}---is a state-of-the-art deep neural network for learning on surfaces. It takes the original meshes together with vertex features as input and combines diffusion on the surface and spatial gradients to predict the class. Sharp et al. suggest using the 3D position or the heat kernel signature of each vertex as features, which we both employed; in addition, we used the inputs to our method as features. To this end, the differential coordinates were averaged over all faces adjacent to the vertex.
            
            The second comparison method, Mesh CNN~\citep{hanocka2019meshcnn}, is another well-known deep neural network for learning from polygonal meshes. It is based on specialized convolution and pooling layers tailored to meshes.
            
            As the third comparison method, we built an Euclidean GNN that closely resembled our proposed network. The manifold GCN block was mimicked using the convolutions from~\cite{kipf2017semisupervised} in combination with linear node-wise layers instead of the diffusion layers with depth-one tMLPs.
            
            As the accuracy measure, we used the ratio of correctly classified subjects. We always split the data set into training, validation, and test sets using a 3:1:1 ratio. After each epoch, the accuracy on the validation set was computed. For each method, the final model was chosen as the first with the highest validation score, and its score on the test set is reported. We repeated this process 100 times for every method, each time with a new (random) data split.

            We also conducted an ablation study to assess the influence of individual parts of the manifold GCN by replacing all tMLPs with the manifold fully connected layer (FCL) from~\cite{chakraborty2020manifoldnet} (\textit{missing tMLPs}), omitting the diffusion layers (\textit{missing diffusion layers}), and omitting the tMLPs (\textit{missing node-wise MLPs}). We also tested the influence of invariance under isometries of the sphere by breaking the equivariance (\textit{missing equivariance}) of the tMLPs. After mapping to the tangent space in the tMLP, each three-vector was therefore viewed as a collection of scalars, and an ordinary MLP with shared weights was applied.

        \subsubsection{Software}
            We used the same software as in the experiment with the synthetic graphs for our network and its training. Additionally, \texttt{PyVista 0.44.1} was used to compute the vertex normals and the volumes of the meshes.
            For DiffusionNet\footnote{\url{https://github.com/nmwsharp/diffusion-net}} and MeshCNN\footnote{\url{https://ranahanocka.github.io/MeshCNN/}}, we used the code that is available online. The Euclidean GNN was built with the graph convolutions provided by \texttt{Jraph}.
        
        \subsubsection{Parameter Settings}
            Based on our empirical observations, we used a learning rate of $10^{-3}$ with exponential decay. We trained with a batch size of 1 for 300 epochs; only our model with differential coordinates as features was trained for 150 epochs due to the observed convergence speed. All computations were performed with double precision.

        \subsubsection{Results and Discussion}

            \begin{table*} [ht]
                \centering
                    \caption{Shape classification results}
                    \label{tab2}
                    \begin{tabular}{ c | c | c | c | c}
                        \textbf{Method} & \textbf{Features} & \textbf{Feature Manifold} & \textbf{\diameter Accuracy} & \textbf{\# param.}\\
                        \hline
                        GCN & Normals & $\mathbb{R}^3$ & $70.9 \pm 10.1$ & 6186 \\
                        Mesh CNN & Geometry & $\mathbb{R}^3$ & $59.2 \pm ~7.3$ & 1320558 \\
                        DiffusionNet & Geometry & $\mathbb{R}^3$ & $61.6 \pm 10.4$ & 116098 \\
                        DiffusionNet & Heat Kernel Signature & $\mathbb{R}^{16}$ & $74.5 \pm ~8.8$ & 116930 \\
                        DiffusionNet & Normals & $\mathbb{R}^3$ & $64.7 \pm ~9.6$ & 116098 \\
                        DiffusionNet & Differential Coords. & $\mathbb{R}^{18}$ & $67.2 \pm ~7.8$ & 117058 \\
                        Ours & Normals & $\textnormal{S}^2$ &  $\underline{75.3} \pm 10.7$ & \textbf{2618} \\
                        Ours & Differential Coords. & $\textnormal{SO}(3)\!\times\!\textnormal{SPD}(3)$ & $\textbf{76.9} \pm ~7.6$ & \underline{6106}
                    \end{tabular}
            \end{table*}

            The results (mean accuracies over 100 splits) of our experiment, along with the number of trainable parameters for each network, are shown in Table~\ref{tab2}. The best-performing method is in bold; the second-best is underlined.
            Remarkably, while our proposed models feature the lowest capacities as gauged by the number of parameters, they achieve the highest accuracies among all tested networks. 
            This suggests that the proposed networks exhibit a better expressivity-capacity ratio than the other methods.

            Our model achieved the highest accuracy when receiving differential coordinates as input. This suggests that there is more information in the differential coordinates representation, and that manifold GCN can leverage it. Nevertheless, our network also achieved a strong result on the more compact representation by normals, thereby utilizing the smallest number of parameters amongst all methods. 
            
            Regarding the other methods, those that take surface embeddings as input---namely, DiffusionNet and Mesh CNN---show the poorest performance. In contrast, using the heat kernel signature as an input feature provides a significant boost in performance, nearly equaling our results. The heat kernel signature captures the \textit{intrinsic} geometry of the surface on various scales, while the normals encode \textit{extrinsic} local curvature information. We thus find that both types of geometric information contain important information. DiffusionNet's results on the normals and differential coordinates prove that our improvements are not due to a different input choice.
            
            The difference between our results and those of the Euclidean GCN shows that being aware of the underlying spherical geometry is vital for the correct identification of discriminating patterns when utilizing the normal field.
            The same applies when using differential coordinates as features, as shown by our results and those of DiffusionNet.

            The results of the ablation study are shown in Table~\ref{tababl}. They verify that all the tested elements contribute to the network performance. While omitting parts of the manifold GCN or breaking the equivariance leads to a large drop in performance, our tMLP compares favourably to the manifold FCL.

\begin{table}[h]
    \centering
        \caption{Ablation Study}
        \label{tababl}
        \begin{tabular}{c | c}
        \textbf{Missing} & \textbf{\diameter Accuracy} \\
        \hline
          tMLPs & $74.5 \pm ~8.1$ \\
          diffusion layers & $70.5 \pm ~9.2$ \\
          node-wise MLPs & $70.5 \pm ~9.3$ \\
          equivariance & $70.8 \pm ~9.1$\\
          None & $\textbf{75.3} \pm 10.7$
        \end{tabular}
\end{table}

\section{Conclusion} \label{sec:conclusions}
    We have presented two new GNN layers for manifold-valued features. They are equivariant under the symmetries of the domain and the feature space of the graph and can be combined to form a highly versatile GNN block that can be used for many deep-learning tasks.
    Unlike existing GNN methods, our layers can handle data from various manifolds. This opens up possibilities for deep learning applications on manifold-valued data that could not be tackled with GNNs.
    
    We applied networks based on our new layers to two graph classification tasks, observing excellent performances: While more widely applicable, they outperformed task-specific state-of-the-art models. 
    The equivariance under symmetries of the feature space enabled a novel graph embedding strategy that exploited the relation between permutations of graph nodes and isometries in hyperbolic and SPD space. The new strategy significantly increased performance over degree-based embeddings, demonstrating a huge potential for geometric representation learning.
    
    Tests with scarce training data suggest that our architectures need less training data, probably because of the additional inductive bias due to the equivariance properties. This might prove helpful in applications where training data is difficult to obtain. An area where this is usually the case is medical shape analysis. Here, we demonstrated a way to utilize sphere-valued features as compact descriptors of shapes to tackle small-sample-size learning tasks in this area. %Nevertheless, interesting future work would be to use shapes as features---a type of non-Euclidean data our new layers can handle.

    Intrinsic architectures that encode the regularities of the data are particularly
    promising for applications with a limited amount of available training data, which is why we focused on them in this work. Clearly, however, evaluating the performance of the proposed network units for different
    tasks, including large-scale benchmarks, is a highly interesting direction
    that we plan to explore in future work. Possible applications of our GNN layers arise, for example, in physics-informed dynamic systems modeling, where it seems interesting to go beyond constant curvature spaces~\citep{sun2025pioneer}; molecule modeling and analysis, where data in the torus and SE(3) appears~\citep{10.5555/3618408.3620080,10.5555/3692070.3693298} and adding our layers might successfully fine-tune current models; and shape analysis~\citep{vonTycowicz_ea2018,Pennec_ea2019_book,AmbellanZachowTycowicz2021}, where the direct application of deep learning to the manifold-valued shapes is underexplored.

    Problems arising with GNNs are over-squashing and bottleneck phenomena~\citep{alon2021on}. They lead to situations in which the transfer of information between distant nodes is tough. A remedy presented by~\cite{chamberlain2021beltrami} is to build layers from diffusion using implicit instead of explicit numerical schemes. This helps to overcome, for example, bottlenecks, as global information is used for an update. This idea can be transferred to our case: Future research can focus on making a diffusion layer based on an implicit (Euler) discretization of the diffusion equation. It also seems interesting to test the use of $p$-Laplacians~\citep{bergmann2018graph} with $p \neq 2$.

    Finally, adding an attention mechanism to the novel layers seems a promising avenue for future work. The attention mechanism has contributed to better performance in various applications, and we anticipate similar benefits for signals that take values on manifolds.

\backmatter

\section*{Declarations}

%Some journals require declarations to be submitted in a standardised format. Please check the Instructions for Authors of the journal to which you are submitting to see if you need to complete this section. If yes, your manuscript must contain the following sections under the heading `Declarations':

\begin{description}
\item[\bf Conflict of interest] The authors declare they have no competing interests.
\item[\bf Funding declaration] Martin Hanik was funded by the Deutsche Forschungsgemeinschaft (DFG, German Research Foundation) under Germany's Excellence Strategy---The Berlin Mathematics Research Center MATH+ (EXC-2046/1, EXC-2046/2, project ID: 390685689).
\item[\bf Ethical approval] Not applicable
\item[\bf Consent to participate]
Not applicable
\item[\bf Consent for publication] Not applicable
\item[\bf Data availability] 
This work relies on data from the open-access  Alzheimer's Disease Neuroimaging Initiative (ADNI)\footnotemark. It can be downloaded from \url{http://adni.loni.usc.edu/}.
\footnotetext{Data collection and sharing for this project was funded by the ADNI (National Institutes of Health Grant U01 AG024904) and DOD ADNI (Department of Defense award number W81XWH-12-2-0012). ADNI is funded by the National Institute on Aging, the National Institute of Biomedical Imaging and Bioengineering, and through generous contributions from the following: AbbVie, Alzheimer's Association; Alzheimer's Drug Discovery Foundation; Araclon Biotech; BioClinica, Inc.; Biogen; Bristol-Myers Squibb Company; CereSpir, Inc.; Cogstate; Eisai Inc.; Elan Pharmaceuticals, Inc.; Eli Lilly and Company; EuroImmun; F. Hoffmann-La Roche Ltd and its affiliated company Genentech, Inc.; Fujirebio; GE Healthcare; IXICO Ltd.; Janssen Alzheimer Immunotherapy Research \& Development, LLC.; Johnson \& Johnson Pharmaceutical Research \& Development LLC.; Lumosity; Lundbeck; Merck \& Co., Inc.; Meso Scale Diagnostics, LLC.; NeuroRx Research; Neurotrack Technologies; Novartis Pharmaceuticals Corporation; Pfizer Inc.; Piramal Imaging; Servier; Takeda Pharmaceutical Company; and Transition Therapeutics. The Canadian Institutes of Health Research is providing funds to support ADNI clinical sites in Canada. Private sector contributions are facilitated by the Foundation for the National Institutes of Health (www.fnih.org). The grantee organization is the Northern California Institute for Research and Education, and the study is coordinated by the Alzheimer's Therapeutic Research Institute at the University of Southern California. ADNI data are disseminated by the Laboratory for Neuro Imaging at the University of Southern California.}
%\item[\bf Materials availability] Not applicable
\item[\bf Code availability] The software components implementing the manifold GCN block have been released as part of the open-source \texttt{Morphomatics}~\citep{Morphomatics} library.
\item[\bf Author contribution] M.H. and C.v.T. wrote the main manuscript text and conducted the experiments. M.H., G.S., and C.v.T. developed the concepts and revised the manuscript. ADNI provides data and is responsible for the design and implementation of data collection.
\end{description}

\begin{appendices}

% The journal wants continuous numbering of figures.
\renewcommand{\thefigure}{\arabic{figure}}  % Remove "A" prefix
\setcounter{figure}{5}  % Continue numbering

\section{Properties of the Graph Diffusion Equation} \label{app:diff_equation}
    This section discusses some properties of the diffusion equation~(\ref{eq:diffusion}). In particular, we prove Theorem \ref{thm:sln}.
    
    For a graph $G = (V, E, w, f)$ with (ordered) vertices $V = \{v_1,\dots,v_n \}$, we define
    \begin{align*}
        \Omega_G := &\{\zb p = (p_1,\ldots,p_n)\in M^n \ |\ \exists f \in \Hsp 
        \\
        & \text{such that} \; f(v_i) = p_i, \; i=1,\ldots,n\}.
    \end{align*}
    This set is open since every maximal neighborhood $D_p \subseteq T_p M$ in which $\exp_p$ is a diffeomorphism is open. Therefore, $\Omega_G$ is a submanifold of $M^n$ (of the same dimension). If $M$ is a Hadamard manifold, then $\Omega_G = M^n$.

     Let $J_i \subseteq \{1,\dots,n\}$ denote the index set of the neighbors of $v_i$. We define the following vector field on $\Omega_G$: 
    %\footnote{Note that $T_{(p_1,\dots,p_n)} \Omega_G \cong %T_{p_1}M \times \dots \times T_{p_n}M$.}
    \begin{align} \label{eq:product_field}
    \zb p \mapsto    \bigg(\sum_{j \in J_i} w(v_i, v_j)  \log_{p_i} p_j \bigg)_{i=1}^n.
    \end{align}
    This vector field is smooth in $\Omega_G$, since, away from the cut locus, $\log_p(q)$ is smooth both as a function in $p$ and $q$; see, for example,~\cite[Prop.\ 18]{Petersen2006}.
    
    %Remember that a curve $\gamma: I \to N$, defined on an %interval $I \subset \mathbb{R}$ and mapping to a smooth %manifold $N$, is an integral curve of a vector field %$X$ on $N$ if $\gamma'(t) = X(\gamma(t))$ for all $t %\in I$. If $0 \in I$, then we say that it starts at %$\gamma(0)$.
    %--------------------------------------------------
    \begin{lemma} \label{lem:integral_curve}
    i)    Equation~\eqref{eq:diffusion} has a solution if and only if there exists an integral curve $\gamma: [0, a) \to \Omega_G$
       with vector field~\eqref{eq:product_field} starting in $(f(v_1),\dots,f(v_n))$. In this case, we have
        \begin{equation} \label{eq:integral_curve}
            \gamma_i(t) = \widetilde{f}(v_i, t), \quad i=1,\dots, n, \quad t \in [0,a).
        \end{equation}
    \\
    ii)   There exists $a \in \mathbb{R}_+$ such that  Equation~\eqref{eq:diffusion} has a unique solution on $[0,a)$, which is moreover smooth.        
    \end{lemma}
    
    \begin{proof}
    i)    Equation~\eqref{eq:diffusion} is an autonomous system of $n$ ordinary differential equations on $M$. Indeed, for $\widetilde{f}(v_i,t) := p_i$, $i=1,\dots,n$, we find
        \begin{equation} \label{eq:hel}
            -\Delta \widetilde{f}(v_i, t) = \sum_{j \in J_i} w(v_i, v_j) \log_{p_i}(p_{j}).
        \end{equation}
        Therefore, if $\widetilde{f}$ solves~\eqref{eq:diffusion}, 
        then we can define the curve $\gamma$ by~\eqref{eq:integral_curve}. Now, Equation~\eqref{eq:hel} implies that the derivative of $\gamma$ at $\zb p$ is given by~\eqref{eq:product_field}, so it is an integral curve of the above vector field.  The other direction works by going backward.
        \\
        ii) If $\gamma$ is smooth, then also $\widetilde f$.
        The result now follows from part i) since integral curves of smooth vector fields that start at an arbitrary point and live for a positive amount of time always exist and are smooth~\citep[Prop.\ 9.2]{Lee2012}.
    \end{proof}
    
        Clearly, if $M$ has a cut locus, then $\Omega_G$ is not (geodesically) complete; when the integral curve hits the boundary, it stops. However, this cannot happen if the data is sufficiently localized, as shown in this appendix.
    
    In the following, a \emph{bounding ball of a graph} is the smallest closed geodesic ball that contains the features of the graph. It exists, for example, if the graph is contained in the maximal neighborhood $D_pM$ of some exponential $\exp_p$. We call a closed geodesic ball convex if it is a normal convex neighborhoods. It is well known that sufficiently small geodesic balls are zu convex~\citep{Petersen2006}. 
    
    Now we can prove Theorem~\ref{thm:sln}, which says that the graph diffusion ``lives forever'' if the bounding ball of the initial features is convex.
    
    \begin{proof}({\bf Proof of Theorem \ref{thm:sln}})
    \\
        
        We know from Lemma~\eqref{lem:integral_curve} that there is a solution $\widetilde{f}$ that lives at least for some time $a > 0$. We must show that $a = \infty$, and we do so by contradiction.

        Assume, therefore, that $a < \infty$. Then, the Escape Lemma~\citep[Lem.\ 9.19]{Lee2012} implies that the corresponding integral curve $\gamma$ leaves every compact subset of $\Omega_G$ before it stops at time $a$. Let $\overline{B}$ be the bounding ball of $G$. Clearly, $\overline{B}^{n}$ is a compact subset of $\Omega_G$. We will show that $\gamma$ cannot leave $\overline{B}^{n}$, leading to the desired contradiction.
        
        To this end, let $v \in V$ be a node whose feature 
        lies on the boundary $\partial\overline{B}$ of $\overline{B}$. Then, because $\sum_{u \sim v} w(v, u) \le 1$ and because connecting geodesics never leave $\overline{B}$, the vector $-\Delta f(v)$ either points to the interior of $\overline{B}$ or along its boundary.
        This is also true for any other $\widehat f \in \Hsp$ for which at least one feature lies on $\partial \overline{B}$ and all others in the interior of $ \overline{B}$.
        Hence, at the boundary $\partial \overline{B}^{n}$, the vectors~\eqref{eq:product_field} point either into the interior of $\overline{B}^{n}$ or are tangent to $\partial \overline{B}^{n}$.
        But then, the integral curve $\gamma$ cannot leave $\overline{B}^{n}$; see~\cite[Lem.\ 9.33]{Lee2012}. 
    \end{proof}
    \begin{example}
        On a $d$-dimensional sphere, any graph contained in an open hemisphere has a convex bounding ball.
    \end{example}

    We now investigate graphs that are ``in equilibrium".
    \begin{definition}
        We call a graph $G=(V,E,w,f)$ stable under diffusion if the corresponding solution of~\eqref{eq:diffusion} is independent of $t$, that is, $\widetilde{f}(v, t) = f(v)$ for all $t \in [0, \infty)$.
    \end{definition}
    As for the corresponding process in the Euclidean space, a graph with constant feature map is stable 
    under diffusion, and many initial configurations will converge to one. (Figure~\ref{fig:diffusion} shows an example of such a case.) Indeed, we conjecture that if a graph's bounding ball is convex, then its features will always converge to a constant point. 
    
    The following example shows that \textit{non}-constant stable functions exist in some manifolds when the support of the features is large enough. 
    
        \begin{figure}[t]
            \centering
            \includegraphics[width=0.3\textwidth]{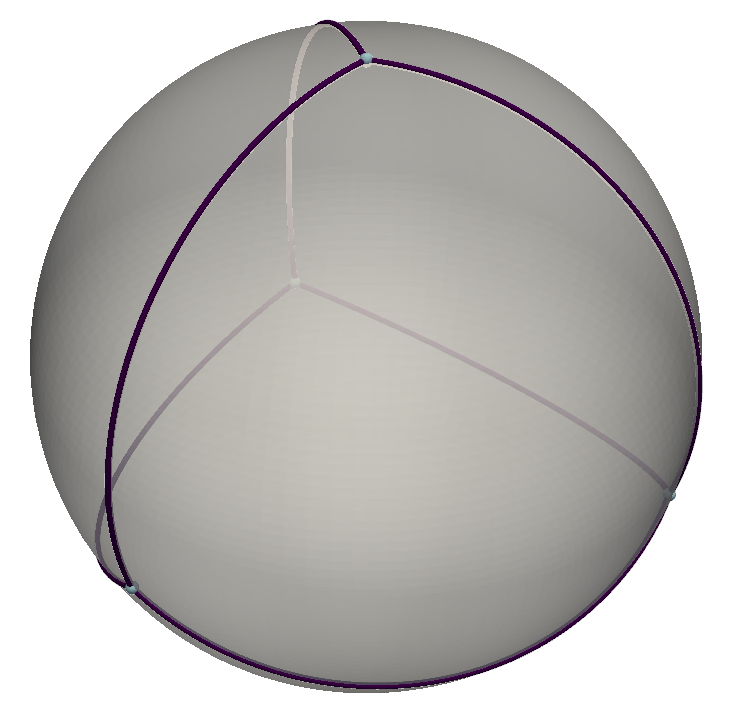}
            \caption{Stable tetrahedron-graph on $\textnormal{S}^2$}
            \label{fig:tetrahedron}
        \end{figure}
        
   \begin{example} \label{ex:tetrahedron}     
        We can inscribe a regular tetrahedron into the $2$-sphere as shown in Figure~\ref{fig:tetrahedron}. Choosing a value of $1/3$ for all edge weights, the right-hand side of Equation~(\ref{eq:graph_laplace}) vanishes at each node. The graph is thus stable under diffusion. 
    \end{example}

    The above example represents a special case of a class of diffusion-stable graphs. Let $(w_i)_{i=1}^m$ be positive weights such that $\sum_{i=1}^m w_i = 1$. Remember that the weighted Fréchet mean of $m$ points $(p_i)_{i=1}^m \in M^m$ is defined by
    \begin{align} \label{eq:wfm}
        \textnormal{wFM}&\Big((w_i)_{i=1}^m, (p_i)_{i=1}^m \Big) \nonumber \\
        &:= \argmin_{p \in M} \sum_{i=1}^m w_i\, \dist(p,p_i)^2.
    \end{align}
    Set $\overline{p}:=\textnormal{wFM}\Big((w_i)_{i=1}^m, (p_i)_{i=1}^m \Big)$.
    Taking the gradient of the sum on the right-hand side of~\eqref{eq:wfm} yields the optimality condition
    \begin{equation}\label{eq:wfm_opt}
        \sum_{i=1}^m -2 w_i \log_{\overline{p}}(p_i) = 0.
    \end{equation}
    This implies the following partial characterization of stable configurations:
    \begin{proposition}
        If for each node $v \in V$ with neighbors $v_1,\dots,v_{m_v} \in V$ we have
        $$f(v) = \textnormal{wFM}\Big(\big(w(v,v_i) \big)_{i=1}^{m_v}, \big(f(v_i) \big)_{i=1}^{m_v} \Big),$$ then the graph is stable under diffusion.
    \end{proposition}    
    \begin{proof}
        The result follows as the right-hand side of the diffusion equation~\eqref{eq:diffusion} vanishes using~\eqref{eq:wfm_opt}.
    \end{proof}
    %Note that, by definition, the proposition presupposes that the weights yield convex combinations. 
    The proposition yields non-constant, stable solutions for positively curved spaces towards which a diffusion process can flow. In particular, as the hippocampi meshes are topological spheres, the graphs used in this experiment had global support on the 2-sphere. We thus think that the trained neural network observed how the graphs diffused towards a globally supported stable configuration similar to the tetrahedron graph from Figure~\ref{fig:tetrahedron}. This makes contact with the cut locus even more unlikely since features do not have to ``cross an equator'' to diffuse toward a constant function.
    
\section{The \texorpdfstring{$\ell$}{l}-step map and normal convex neighborhoods} \label{app:l_step}
    In this appendix, we give some theoretical results on the behavior of the $\ell$-step map \eqref{l-step} for localized data.
    \begin{proposition} \label{prop:convex_ball}
        Let $M$ be a Riemannian manifold, and let $G = (V, E, w, f)$ be a graph with positive weights and features $f \in \Hsp$ such that the bounding ball $\overline{B}$ is convex. Assume $\sum_{u \sim v} w(v, u) \le 1$ for all $v \in V$. Then, there is a maximal $a > 0$ such that for all $t \le a$ 
        $$\step^{\ell}_{t,\theta}(f)(V) \subset \overline{B}.$$
    \end{proposition}
    
    \begin{proof}
        Let $v \in V$ be a node whose feature $f(v)$ lies on the boundary of $\overline{B}$. Then, because $\sum_{u \sim v} w(v, u) \le 1$, the vector $-\Delta f(v)$ either points to the interior of $\overline{B}$ or along its boundary. Thus, there is a maximal $t_v > 0$ such that $\exp_{f(v)}(-t\Delta f(v)) \in \overline{B}$ for all $t \le t_v$. Obviously, such a number also exists for vertices whose features are in the interior of $\overline{B}$. Hence, 
        $a = \max_{v \in V} t_v$.
    \end{proof}
    
    We directly get two corollaries.
    
    \begin{corollary} \label{cor:1}
        Let the assumptions of  Proposition~\ref{prop:convex_ball} be fulfilled.
        Then there exists $a > 0$ such that for all $t \le a$ 
        $$\step^{\ell}_{t,\theta}(f)(v) \in \overline{B}.$$
        In particular, every step of the $\ell$-step map is well-defined.
    \end{corollary}
    
    By Corollary \ref{cor:1}, a diffusion layer cannot run into trouble with a cut locus, 
    if the bounding ball of an input graph is convex and 
    the time parameter $t$ is smaller than some $a >0$.
    The following corollary tells us yet more about the behavior of the $\ell$-step map.
    
    \begin{corollary}
       Let the assumptions of  Proposition~\ref{prop:convex_ball} be fulfilled. Then there exists $a > 0$ such that for all $t < a$ 
        \begin{align*}
            \max_{v, u \in V} \dist \big(\step^{\ell}_{t,\theta}(f)(v),\ &\step^{\ell}_{t,\theta}(f)(u) \big) \\
            &< \max_{v, u \in V}\dist \big(f(v), f(u) \big).
        \end{align*}
    \end{corollary}
    
    In machine learning applications, we expect that it is preferable when the graphs shrink in diameter (or, at least, do not expand) since most deep learning layers show such a ``contractive'' behavior; see~\cite{chakraborty2020manifoldnet} for a more in-depth discussion on this topic. 
    The corollaries suggest that the diffusion times of a diffusion layer should be initialized 
    close to zero, and the weights of a graph should be normalized, as proposed in Section~\ref{sec:GCN}.

\end{appendices}

\bibliography{our_bibliography}% 

\end{document}